\documentclass{article}

\usepackage[nonatbib,final]{ewrl_2023}
\usepackage[round]{natbib}

\usepackage{enumitem}
\usepackage[utf8]{inputenc} 
\usepackage[T1]{fontenc}    
\usepackage[hidelinks]{hyperref}       
\usepackage{url}            
\usepackage{booktabs}       
\usepackage{amsfonts}       
\usepackage{nicefrac}       
\usepackage{microtype}      
\usepackage{xcolor}         

\usepackage[english]{babel}
\usepackage{amsmath,amssymb,mathrsfs}
\usepackage[capitalize]{cleveref}
\usepackage{bbm}
\usepackage{parskip}
\usepackage{algorithm}
\usepackage[noend]{algorithmic}
\usepackage{amsthm}
\usepackage{thmtools}
\usepackage{thm-restate}
\usepackage{minitoc}

\newtheorem*{theorem*}{Theorem}
\newtheorem*{corollary*}{Corollary}
\newtheorem{definition}{Definition}
\newtheorem{lemma}{Lemma}

\newtheorem{assumption}{Assumption}

\newtheorem{corollary}{Corollary}
\newtheorem{remark}{Remark}

\newcommand\tab[1][1cm]{\hspace*{#1}}
\usepackage{ulem}
\usepackage{array}

\newcommand\St{\mathcal{S}}
\newcommand\A{\mathcal{A}}
\newcommand\Xc{\mathcal{X}}
\newcommand\E{\mathbb{E}}
\newcommand\N{\mathcal{N}}
\newcommand\M{\mathcal{M}}

\newcommand\indic{\mathbbm{1}}
\newcommand\R{\mathbb{R}}
\newcommand\Reg{\mathcal{R}}
\newcommand\xbar{\bar{x}}
\newcommand\pibar{\bar{\pi}}

\newcommand\pt{\tilde{p}}
\newcommand\ct{\tilde{c}}
\newcommand\dt{\tilde{d}}
\newcommand\Dt{\tilde{D}}

\newcommand\Ltk{\tilde{L}_{k}}
\newcommand\pb{\bar{p}}
\newcommand\cb{\bar{c}}
\newcommand\db{\bar{d}}
\newcommand\nkm{n_h^{k-1}(s,a) \vee 1}
\newcommand\curly[1]{\{#1\}}
\newcommand\Simplex[1]{\Delta\left(#1\right)}
\newcommand\AlgName{\textsc{OptAug-CMDP}}
\newcommand\kPrime{\max \left \{  \frac{S^2 A H^3}{(1-\nu) \gamma}, \frac{\mathcal{N}SAH^4}{(1-\nu)^2 \gamma^2} \right \}}

\usepackage{natbib}

\title{Cancellation-Free Regret Bounds for Lagrangian Approaches in Constrained Markov Decision Processes}

\author{%
  Adrian Müller\\
  Department of Computer Science\\
  ETH Zürich\\
  \texttt{adrian.mueller@inf.ethz.ch} \\
  \And
  Pragnya Alatur \\
  Department of Computer Science \\
  ETH Zürich and ETH AI Center \\
  \texttt{pragnya.alatur@ai.ethz.ch} \\
  \AND
  Giorgia Ramponi \\
  Department of Computer Science \\
  ETH Zürich and ETH AI Center \\
  \texttt{giorgia.ramponi@inf.ethz.ch} \\
  \And
  Niao He \\
  Department of Computer Science \\
  ETH Zürich and ETH AI Center \\
  \texttt{niao.he@inf.ethz.ch} \\
}

\begin{document}

\maketitle

\begin{abstract}
    Constrained Markov Decision Processes (CMDPs) are one of the common ways to model safe reinforcement learning problems, where constraint functions model the safety objectives. Lagrangian-based dual or primal-dual algorithms provide efficient methods for learning in CMDPs. For these algorithms, the currently known regret bounds in the finite-horizon setting allow for a \textit{cancellation of errors}; one can compensate for a constraint violation in one episode with a strict constraint satisfaction in another. However, we do not consider such a behavior safe in practical applications.

    \looseness -1In this paper, we overcome this weakness by proposing a novel model-based dual algorithm \AlgName~for tabular finite-horizon CMDPs. Our algorithm is motivated by the augmented Lagrangian method and can be performed efficiently. We show that during $K$ episodes of exploring the CMDP, our algorithm obtains a regret of $\tilde{O}(\sqrt{K})$ for both the objective and the constraint violation. Unlike existing Lagrangian approaches, our algorithm achieves this regret without the need for the cancellation of errors.
\end{abstract}


\section{Introduction}

In classical reinforcement learning \citep[RL,][]{sutton2018reinforcement}, the goal is to learn an optimal policy when interacting with an unknown Markov decision process \citep[MDP,][]{bellman1957}. In MDPs, an agent aims to minimize the expected cumulative cost incurred during an episode. However, the learned policy must often adhere to certain safety constraints in practical scenarios. For example, when navigating a car on a race track, one would want to avoid crossing the boundaries of the track too often. Such safety requirements are commonly modeled via constrained Markov decision processes \citep[CMDPs,][]{altman1999constrained}. We consider the problem of learning an optimal feasible policy in a CMDP. That is, the goal of the agent is to minimize the cost while satisfying the constraints\footnote{i.e., being feasible for the CMDP, which we also refer to as being \textit{safe}}. Since the CMDP is unknown, we formalize these desiderata by considering the regret with respect to an optimal feasible solution for the cost and the constraint violation, respectively. 

\looseness -1Importantly, we do not consider it sufficient to provide an agent whose cumulative cost suboptimality and cumulative constraint violation are sublinear. This is because an agent can have a negative constraint violation (by being very safe but incurring a higher cost than an optimal safe policy) or a positive constraint violation (by being unsafe but incurring a lower cost than an optimal safe policy). Thus, terms from these two cases can cancel each other out, which we refer to as the so-called \textit{cancellation of errors} \citep{efroni2020exploration}. An agent for which these cumulative terms are sublinear may violate the safety constraints heavily during learning by oscillating around an optimal safe policy. While such a method converges on average to an optimal safe policy\footnote{Here, we refer to the value functions for the underlying CMDP.}, it neither allows for directly extracting an optimal feasible policy nor does it guarantee safety during learning. We consider a stronger notion of regret, which overcomes this issue by considering the sum of the \textit{positive parts} of the error terms instead. As pointed out by \citet{efroni2020exploration}, it is of major theoretical interest whether Lagrangian approaches can achieve sublinear bounds for this notion of regret.

\looseness -1The approaches to learning CMDPs are split into linear programming (LP) and Lagrangian approaches\footnote{i.e., dual and primal-dual algorithms} \citep{altman1999constrained,efroni2020exploration}. While LP-based algorithms generally allow for sublinear regret bounds without the need for cancellations \citep{efroni2020exploration}, they can be expensive when dealing with large state-action spaces. In contrast, in Lagrangian methods, we can solve the optimization problem arising in each episode using dynamic programming (DP), offering a computational benefit over solving LPs. However, the currently known bounds for Lagrangian approaches only concern a weaker form of regret that allows for the aforementioned cancellation of errors. As \citet{efroni2020exploration} pointed out, this is due to the underlying optimization methods rather than a weakness of the analysis. The main goal of this paper is to provide a Lagrangian-based algorithm that guarantees sublinear regret without the cancellation of errors. To achieve this, the key problem we solve is stopping the agent from oscillating around an optimal safe policy. Our contributions can be summarized as follows:
\begin{itemize}[leftmargin=*]
    \item We propose a novel model-based dual algorithm, \AlgName, for learning an optimal feasible policy in an unknown CMDP (\cref{sec:algos}). The algorithm is split into a model pre-training phase and an optimistic exploration phase motivated by the augmented Lagrangian method.
    \item We show that a sub-problem required for \AlgName~can be reformulated as a convex optimization problem. We provide an efficient algorithm to solve it (\cref{sec:algos}) despite the non-linearity introduced by considering the augmented Lagrangian.
    \item We prove that with high probability, during $K$ episodes \AlgName~achieves regrets for the cost and the constraint violations of $\tilde{O}(\sqrt{K})$ when only highlighting the dependency on $K$. Notably, we achieve this bound for the stronger notion of regret, which does not allow for the cancellation of errors. This partly settles the open problem posed by \citet{efroni2020exploration}.
\end{itemize}


\subsection{Related Work}

The most relevant foundation for our work is the work by \citet{efroni2020exploration}, which reviews model-based algorithms for CMDPs and establishes regret bounds for them. The authors analyze the LP-based algorithms \textsc{OptCMDP} and \textsc{OptCMDP-bonus} that achieve sublinear regret without cancellations. However, the Lagrangian-based algorithms they analyze, \textsc{OptDual-CMDP} and \textsc{OptPrimalDual-CMDP}, only achieve sublinear regret with cancellations. This is because the oscillatory behavior of dual and primal-dual descent methods prevents the individual iterates from being approximately feasible. Therefore, the authors pose the open question of whether one can devise Lagrangian-based algorithms that do not suffer from this issue.

The majority of relevant work providing guarantees for Lagrangian approaches to CMDPs is concerned with model-free primal-dual algorithms \citep{ding2020natural, bai2022achieving, ding2022convergencepol, ding2022convergencenat} or model-based dual algorithms \citep{liu2021learning, efroni2020exploration}. However, in both cases, the existing literature does not address the issue of the cancellation of errors when exploring the CMDP and thus does not provide a method for \textit{safely} finding an optimal feasible policy.\footnote{That is, there are no guarantees on the constraint violations of the individual iterates for these methods.} While there is work on analyzing different forms of regularization to the Lagrangian-based algorithms, their guarantees either require the cancellation of errors \citep{liu22global,li2021faster} or assume access to exact value functions \citep{ying2022dual}. \citet{moskovitz2023reload} propose a first approach to address the cancellation of errors by replacing gradients with their optimistic gradient counterparts in well-known Lagrangian-based RL algorithms. While they show the empirical success of their methods, their theoretical analysis only covers a hypothetical algorithm with implicit updates and requires full knowledge of the CMDP. \citet{stooke2020responsive} address the underlying problem of oscillations of Lagrangian methods for CMDPs via PID control in the context of deep RL, providing experimental successes but no guarantees. Thus, to the best of our knowledge, none of the existing works address the open question of \citet{efroni2020exploration} in the setup of an unknown CMDP. 

While there are mentions of using the augmented Lagrangian method for CMDPs \citep{li2021augmented, lu2022single, krishnamurthy2003self, krishnamurthy2011real, krishnamurthy2012gradient}, all such works are concerned with research questions rather different from ours. The only one similar to ours is that of \citet{li2021augmented}. The authors propose a surrogate reward inspired by the augmented Lagrangian to promote safety during learning. However, their method significantly differs from ours as it is concerned with \textit{instantaneous} constraints in an infinite-horizon CMDP. Moreover, their analysis only shows that an optimal policy for their surrogate MDP is optimal for the original CMDP (under certain assumptions).


\section{Background and Problem Formulation} \label{sec:notation}


\textbf{Notation:} For any $n\in\mathbb{N}$, we use the short-hand notation $[n]$ to refer to the set of integers $\curly{1, \dots ,n}$. For any finite set $X$, we denote by $\Simplex{X}$ the probability simplex over $X$, i.e., $\Simplex{X} = \{v \in [0,1]^{X} | \sum_{x \in X} v(x) = 1\}$. For $a\in\R$, we set $[a]_+ := \max\{0, a\}$ to be the positive part of $a$. For a vector $b\in\R^n$, we write $[b]_+$ for the vector whose entries are the positive parts of the corresponding entries of $b$. Similarly, for two vectors $a, b\in\R^n$, we write $a\leq b$ as a short-hand for $a_i \leq b_i$, for all $i\in[n]$. Throughout the paper, we denote the Euclidean norm by $\| \cdot \|$.

\looseness -1 We define a finite-horizon CMDP as a tuple $\mathcal{M} = (\St, \A, H, p, c, (d_i)_{i\in[I]}, (\alpha_i)_{i\in[I]}, s_1)$ with the following components. $\St$ and $\A$ are the state and action space, respectively, and $H>0$ denotes the horizon. Every episode consists of $H$ steps, starting from the initial state $s_1\in\St$. At every step $h\in[H]$, $p_h(s' | s,a)$ denotes the probability of transitioning to state $s'$ if the current state and action are $s$ and $a$. Moreover, $c_h \colon \St\times\A\rightarrow [0,1]$ denotes the objective cost function at step $h$. For $i\in [I]$, $d_{i,h} \colon \St\times\A\rightarrow [0,1]$ refers to the cost function of the $i$-th constraint at step $h$, and $\alpha_i\in [0,H]$ denotes the threshold for the $i$-th constraint. We assume the state and action space are finite, with cardinalities $S$ and $A$, respectively. Furthermore, we assume the agent does not know the transition probabilities, objective costs, or constraint costs beforehand. Whenever the agent takes an action $a$ in state $s$ at time $h$, it observes costs sampled from random variables $C_h(s,a) \in [0,1]$ and $( D_{i,h}(s,a))_{i\in [I]} \in [0,1]^I$ such that $\E[C_h(s,a)] = c_h(s,a)$ and $\E[D_{i,h}(s,a)] = d_{i,h}(s,a)$, for all $i\in [I]$. The agent interacts with the CMDP by playing a policy $\pi=(\pi_h)_{h\in[H]}$, meaning that if in state $s$ at step $h \in [H]$, the agent samples its next action from $\pi_{h}(\cdot|s) \in \Simplex{\A}$. For an arbitrary cost function $l = (l_h)_{h\in [H]}$ and transition probabilities $p' = (p_h')_{h\in [H]}$, the expected cumulative cost under policy $\pi$ is measured by the value function defined as follows:
\begin{align*}
    V^{\pi}(l, p') :=& \E \bigg[ \sum_{h=1}^H l_h(s_h, a_h)  \mid s_1, \pi, p' \bigg],
\end{align*}
where $(s_h,a_h)$ denotes the state-action pair at step $h$ under transitions $p'$ and policy $\pi$. We fix an optimal solution of the CMDP, given by a policy $\pi^*$, defined as follows:
\begin{equation}
    \pi^* \in \arg\min_{\pi \in \Pi} ~~~~ V^{\pi}(c, p)  ~~~~
        \text{s.t.}  ~~~~ V^{\pi}(d_i, p) \leq \alpha_i ~~~~(\forall i \in [I]). \label{opt:cmdp}
\end{equation}
For brevity, we write $V^{\pi}((l_i)_{i\in[I]}, p') := (V^{\pi}(l_1, p'), \dots, V^{\pi}(l_I, p'))^T \in \R^{I}$ in the presence of $I$ different cost functions $l_i = (l_{i,h})_{h\in[H]}$ ($i \in [I]$), and $\alpha := (\alpha_1, \dots, \alpha_I)^T \in \R^I$. Furthermore, we denote by $\Pi:= \{\pi=(\pi_h)_{h\in[H]}|\pi_h:\St\rightarrow \Simplex{\A}\}$ the entire policy space.


\textbf{Strong duality and dual methods:} \citet{altman1999constrained,paternain2019strongduality} proved that CMDPs possess the \textit{strong duality} property; i.e., given a feasible CMDP $\M$, the following relation holds:
\begin{equation}
\label{eq:minmaxformulation}
    V^{\pi^*}(c,p) = 
    \underbrace{\min_{\pi\in \Pi} \max_{\lambda\in \R_+^I} ~ \mathcal{L}(\pi, \lambda)}_{\text{Primal problem}} = \underbrace{\max_{\lambda\in \R_+^I} \min_{\pi\in \Pi} ~ \mathcal{L}(\pi, \lambda)}_{\text{Dual problem}},
\end{equation}
where $\mathcal{L}(\pi,\lambda) := V^{\pi}(c,p) + \lambda^T \left(V^{\pi}((d_i)_{i\in[I]}, p) - \alpha \right)$ denotes the Lagrangian. The strong duality property gives theoretical justification to dual methods \citep{altman1999constrained,efroni2020exploration, paternain2019strongduality}. These methods are popular, as the dual problem can be solved via a sequence of (extended\footnote{If the CMDP is unknown, backward induction involves an extra optimization step over the possible transitions \citep{jin2019learning}.}) \textit{unconstrained} MDPs, each of which can be solved efficiently via DP (as opposed to using LPs for solving a sequence of CMDPs, for which the Bellman optimality principle does not hold). 

In this work, we solve the min-max problem in \cref{eq:minmaxformulation} using the augmented Lagrangian method. This is beneficial since the analysis of the augmented Lagrangian method allows for convergence guarantees concerning the last iterate and not just the averaged iterates. Since the occurring sub-problems are not MDPs anymore, we justify in \cref{sec:algos} how they can still be solved efficiently by leveraging a Frank-Wolfe scheme and DP. We are now ready to state the main problem formulation of our work.


\textbf{Problem formulation:} In our setting, the agent interacts with the unknown CMDP over a fixed number of $K>0$ episodes. In every episode $k\in [K]$, the agent plays a policy $\pi_k\in\Pi$ and its goal is to (simultaneously) minimize its regrets, defined as follows:
\begin{align*}
    \Reg(K; c) &:= \sum_{k\in[K]} [ V^{\pi_k}( c, p) - V^{\pi^*}(c,p) ]_+, \tag{Objective strong regret}\\
    \Reg(K; d) &:= \max_{i\in [I]} \sum_{k\in[K]} \left[ V^{\pi_k}(d_i, p) - \alpha_i \right]_+. \tag{Constraint strong regret}
\end{align*}
For simplicity, we will write \textit{regret} when referring to the strong regret throughout the paper. As we pointed out, existing works on Lagrangian-based algorithms \citep{liu2021learning, efroni2020exploration, bai2022achieving, ding2022convergencepol, ding2022convergencenat} only prove sublinear guarantees on a \textit{weaker} notion of regret, defined as follows:
\begin{align*}
        \Reg_{\pm}(K; c) &:= \sum_{k\in[K]} ( V^{\pi_k}(c,p) - V^{\pi^*}(c,p) ), \tag{Objective weak regret}\\ \Reg_{\pm}(K; d) &:= \max_{i\in [I]} \sum_{k\in[K]} \left( V^{\pi_k}(d_i,p) - \alpha_i \right). \tag{Constraint weak regret}
\end{align*}
The weak regrets allow for the aforementioned cancellation of errors; i.e., even if they are sublinear in $K$, the agent can continue compensating for a constraint violation in one episode with strict constraint satisfaction in another. On the other hand, a sublinear bound on the stronger notion of regret guarantees that the agent achieves a low constraint violation in most episodes.\footnote{Indeed, fix $\epsilon > 0$ and suppose $\Reg(K;d) \leq \tilde{O}(\sqrt{K})$. Then there exist at most $\tilde{O}(\sqrt{K} / \epsilon)$ episodes with a constraint violation of at least $\epsilon$. In other words, only a small fraction $\tilde{O}(1/(\epsilon\sqrt{K}))$ of the iterates is not $\epsilon$-safe. In comparison, this is by no means guaranteed by a sublinear bound on $\Reg_{\pm}(K;d)$.} While this is crucial for practical applications, providing a bound for the strong regrets is strictly more challenging than for the weaker notion.


\section{Algorithm and Main Result} \label{sec:algos}

\looseness -1In this section, we introduce our algorithm \AlgName~(see \cref{algo:pp}) and state its regret guarantees in \cref{thm:optaugrlpp-regret}. In \AlgName, the agent interacts with the unknown CMDP over a fixed number of $K>0$ episodes. To encourage exploration of the CMDP, the agent follows the well-known \textit{optimism in the face of uncertainty} principle \citep{auer2008near} and builds an optimistic estimate of the CMDP in every episode $k\in [K]$. That is, in every episode $k\in [K]$, the agent builds optimistic estimates $\ct_k$ for the objective cost $c$, optimistic estimates $\dt_{i,k}$ for the constraint costs $d_i$, and a set of plausible transition probabilities $B_k^p$, which we specify in the following paragraph.


\textbf{Optimistic estimates:} Let $n_h^{k-1}(s,a) := \sum_{l=1}^{k-1} \indic_{\{ s_h^l=s,~ a_h^l = a \}}$ count the number of times that the state-action pair $(s,a)$ is visited at step $h$ before episode $k$. Here, ($s_h^l$, $a_h^l$) denotes the state-action pair visited at step $h$ in episode $l$. First, we compute the empirical averages of the cost and transition probabilities as follows:
\begin{equation*}
\begin{gathered}
\begin{aligned}
    \cb_h^{k-1}(s,a) :=& \frac{\sum_{l=1}^{k-1} C_h^{l}(s,a) \indic_{\{ s_h^l=s,~ a_h^l = a \}}}{\nkm}, \\
    \db_{i,h}^{k-1}(s,a) :=& \frac{\sum_{l=1}^{k-1} D_{i,h}^{l}(s,a) \indic_{\{ s_h^l=s,~ a_h^l = a \}}}{\nkm} \quad (\forall i\in [I]), \\
    \pb_h^{k-1}(s'|s,a) :=& \frac{\sum_{l=1}^{k-1}, \indic_{\{ s_h^l=s,~ a_h^l = a,~ s_{h+1}^l=s' \}}}{\nkm},
\end{aligned}
\end{gathered}
\end{equation*}
where $a\vee b := \max\{a,b\}$. With this, we define the optimistic costs and the set of plausible transition probabilities as
\begin{flalign}
    \ct_{k,h}(s,a) &:= \cb_h^{k-1}(s,a) - \beta^c_{k,h}(s,a), \nonumber\\
    \dt_{i,k,h}(s,a) &:= \db_{i,h}^{k-1}(s,a) - \beta_{i,k,h}^{d}(s,a) \label{eq:opt-c} \quad (\forall i\in [I]),\\
    B_{k,h}^{p}(s,a) &:= \{ \pt_{h}(\cdot | s,a) \in \Simplex{\St} \mid \forall s' \in \St \colon |\pt_h(s' | s,a) - \bar{p}_h^{k-1}(s' | s,a) | \leq \beta_{k,h}^{p}(s,a,s') \}, \nonumber\\
    B_k^p &:= \{ \pt \mid \forall s,a,h \colon \pt_h(\cdot | s,a) \in B_{k,h}^p(s,a) \} \nonumber.
\end{flalign}
Here, $\beta^c_{k,h}(s,a) = \beta_{i,k,h}^{d}(s,a) > 0$ denote the exploration bonus for the costs and $\beta_{k,h}^{p}(s,a,s') > 0$ denotes the confidence threshold for the transitions. For any $\delta \in (0,1)$, we specify the correct values for those quantities in \cref{sec:regret-prelim-omit} to obtain our regret guarantees with probability at least $1- \delta$. In the next paragraph, we describe how the agent computes its policy in episode $k$.


\textbf{Policy update:}
Given the optimistic CMDP at episode $k$, we derive the next policy $\pi_{k}$ using a scheme motivated by the augmented Lagrangian method (cf. \cref{eq:primal-step,eq:dual-step}). At the end of this section, we explain how we can perform the optimization step in \cref{eq:primal-step} efficiently, up to a specified accuracy $\epsilon_k$. For now, we treat this part of the algorithm as a black-box subroutine.

\looseness -1Optimistic exploration alone with the augmented Lagrangian, however, is insufficient to obtain sublinear regret guarantees for our algorithm. For technical reasons, our analysis also requires the optimistic CMDPs with costs $\ct_k$, $(\dt_{i,k})_{i\in [I]}$ and transitions $\pt_{k} \in B_k^p$ (cf. \cref{eq:primal-step}) to be \textit{strictly} feasible, in every episode $k\in [K]$. Our analysis in \cref{sec:term2} explains the need for this technical assumption. To address this issue, we propose a pre-training phase \textit{before} the optimistic exploration phase, which we describe in the following paragraph.


\textbf{Pre-training phase:} In this phase, the agent repeatedly executes a fixed policy $\pibar$ for $K'\leq K$ episodes. The policy $\pibar$ must be \textit{strictly} feasible for the true CMDP, which we formally state in the following assumption.
\begin{restatable}[Strictly feasible policy]{assumption}{assmildslater} \label{ass:mild-slater}
    We have access to a policy $\pibar$ such that $V^{\pibar}(d_i, p) < \alpha_i$ for all $i\in[I]$. Furthermore, we assume that the slack $\gamma$, defined below, is known\footnote{In other words, there is a Slater point for the constraint set of the true CMDP. Note that knowing a lower bound instead of the exact slack $\gamma$ is sufficient as well.}:
    \begin{align*}
        \gamma := \min_{i \in [I]} ~ (\alpha_i - V^{\pibar}(d_i, p)) \in (0, H].
    \end{align*}
\end{restatable}
\looseness -1Note that this is stronger than only assuming the \textit{existence} of a strictly feasible policy. However, making this assumption is realistic in many practical setups \citep{liu2021learning, bura2022dope}. For example, in the case of a race car that should not exceed the boundary of a track, it would be sufficient to have access to the policy of a car that strictly stays within the boundaries but may be arbitrarily slow. To address the technical issue mentioned earlier, we need to set $K'$ such that the following condition holds for some $\nu \in (0,1)$, with high probability:
\begin{align*}
    V^{\pibar}(\dt_{i,k}, \pt_{k}) \leq \alpha_i - \nu \gamma  ~~~~ (\forall i\in [I] ~ \forall k\in \left\{K', \dots ,K \right\}),
\end{align*}
where $\pt_{k}$ is defined by the update in \cref{eq:primal-step}.\footnote{For $k=K'$, we just take any $\pt_{K'} \in B_{K'}^p$} In particular, the fixed policy $\pibar$ is strictly feasible for the optimistic CMDP at every episode $k \geq K'$. Indeed, if the agent plays $\pibar$ for the first $K'$ episodes of the algorithm with a large enough constant $K'$, then for all future episodes, the constraint value function of $\pibar$ under the estimated model is close to the constraint value function of $\pibar$  under the true model. Thus, we can ensure the above condition. Leveraging an adaption of an on-policy error bound (see \cref{sec:onp-bounds}), we prove that it is sufficient to set $K'$ as follows:
\begin{restatable}[]{lemma}{proppretraining} \label{prop:pretraining}
    Suppose that Assumption \ref{ass:mild-slater} holds, i.e., the agent has access to a strictly feasible $\pibar$ and its slack $\gamma>0$. Fix any $\nu \in (0,1)$, and suppose the agent executes $\pibar$ for $K' = \tilde{O} \left( \kPrime \right)$ episodes, where $\N := \max_{s,a,h} |\{ s' \mid p_h(s' | s,a) > 0 \}|$ denotes the maximum number of transitions. Then, if the agent updates the optimistic CMDP based on the observations from those episodes (cf. \cref{eq:opt-c}), with probability at least $1-\delta$ the following condition is satisfied for every $k\in \left\{K', \dots ,K\right\}$:
    \begin{align*}
        V^{\pibar}(\dt_{i,k}, \pt_{k}) \leq \alpha_i-\nu\gamma ~~~~ (\forall i\in [I]).
    \end{align*}
\end{restatable}

We present the resulting \AlgName~algorithm in \cref{algo:pp}.
\begin{algorithm}[H]
	\begin{algorithmic}
        \REQUIRE $K$ (total number of episodes), $K'\leq K$ (number of pre-training episodes), $(\eta_k)_{k\geq K'+1}$ (step sizes), $(\epsilon_k)_{k\geq K'+1}$ (accuracies), $\pibar$ (strictly feasible policy), $\alpha$ (constraint thresholds), $\lambda_{K'+1}:=0 \in \R^I$ \vspace{0.2cm}
        \STATE \textit{// Phase 1: Pre-training the model}
        \FOR{$k = 1, \dots, K'$} 
            \STATE Play policy $\pi_{k} = \pibar$, update estimates of the costs $\ct_{k+1}$, $(\dt_{i,k+1})_{i\in[I]}$ and transitions $B_{k+1}^p$ (\cref{eq:opt-c}).
        \ENDFOR \vspace{0.2cm}
        \STATE \textit{// Phase 2: Optimistic exploration with pre-trained model}
        \FOR{$k=K'+1,\dots, K$} 
            \STATE{Update policy (by finding $\pi_{k}$, $\tilde{p}_{k}$ such that the objective is $\epsilon_k$-close to the minimum):}
            \begin{align}
                \pi_{k}, \tilde{p}_{k} := \arg\min_{\substack{\pi \in \Pi \\ p' \in B_k^p}} \left( V^{\pi}(\ct_k, p') + \frac{1}{2\eta_k}  \| [ \lambda_k + \eta_k ( V^{\pi}((\dt_{i,k})_{i\in[I]}, p') - \alpha ) ]_+  \|^2 \right) \label{eq:primal-step}
            \end{align}
            \STATE{Update dual variables:}
            \begin{align}
                \lambda_{k+1} := [ \lambda_{k} + \eta_k (  V^{\pi_{k}}((\dt_{i,k})_{i\in[I]}, \pt_{k}) - \alpha ) ]_+ \label{eq:dual-step}
            \end{align}
            \STATE Play $\pi_{k}$, update estimates of the costs $\ct_{k+1}$, $(\dt_{i,k+1})_{i\in[I]}$ and transitions $B_{k+1}^p$ (\cref{eq:opt-c}).
        \ENDFOR
		\caption{\AlgName}
    \label{algo:pp}
	\end{algorithmic}	
\end{algorithm}

We are now ready to state the regret guarantees for \AlgName.
\begin{restatable}{theorem}{thmoptaugrlppregret}\label{thm:optaugrlpp-regret}
    Suppose that Assumption \ref{ass:mild-slater} holds, let $\delta \in (0,1)$ and $\nu > 0$. Then there exist $K' = \tilde{O} \left( \kPrime \right)$ and $\eta_k$, $\epsilon_k$ such that with probability at least $1 - \delta$, \AlgName~achieves a total regret of
    \begin{align*}
        \Reg(K; c) &= \tilde{O} \left( \sqrt{\mathcal{N}SAH^4 K} + S^2 A H^3  + K'H \right),\\
        \Reg(K;d) &= \tilde{O} \left( \sqrt{\mathcal{N}SAH^4 K} + S^2 A H^3 \right).
    \end{align*}
\end{restatable}
We remark that to achieve this bound, using step sizes $\eta_{K'+k} = \Theta(k^{2.5})$ and accuracies $\epsilon_{K'+k} = \Theta(1/\eta_{K'+k})$ (when only highlighting the dependency on $k$) is sufficient, as we discuss in \cref{sec:main-res-omit}. 


\textbf{Comparison with \textsc{OptDual-CMDP}:} 
Crucially, our bound holds for the stronger notion of regret. In contrast, the one for the related \textsc{OptDual-CMDP} algorithm (see \cref{sec:optdual}) by \citet{efroni2020exploration} only concerns the weak regret, which allows for the cancellation of errors. Apart from this, the bound we obtain is similar in spirit. However, our regret bound does not depend on the number $I$ of constraints up to polylogarithmic factors. Moreover, we get slightly different (but not worse) constants and the constant extra term $K' H$ due to the model pre-training phase. In addition, it is important to note that we can choose $\eta_k$, $\epsilon_k$ in terms of $\gamma$ (Assumption \ref{ass:mild-slater}) such that in the leading term of the regret bound, there is no dependency on $\min_{i\in[I]} ( \alpha_i - V^{\pibar}(d_i, p))$, as opposed to \textsc{OptDual-CMDP}.


\textbf{Solving the inner problem:} 
We now elaborate on the subroutine for solving the optimization problem in \cref{eq:primal-step} that defines the policy update in \cref{algo:pp}. Importantly, we can reformulate \cref{eq:primal-step} as a constrained optimization problem that is convex in the state-action-state occupancy measure (see \cref{sec:fw-inner}). However, the resulting problem is neither an LP nor an extended MDP (due to the nonlinear objective), which prevents solving it with a single DP or LP solver call. Albeit related, it is not a standard convex RL problem either (due to the additional optimization over the constraint set $B_k^p$). Moreover, computing projections onto the high-dimensional domain is prohibitive, making it impossible to run projected gradient descent. 

The projection-free method we propose in \cref{sec:pseudo-inner} overcomes this difficulty by combining a Frank-Wolfe scheme with DP in a sequence of (extended) MDPs. In every iteration of this inner method, we consider the linear minimization step needed for a Frank-Wolfe iteration. When switching back to optimization over $\Pi$ and $B_k^p$, we can then perform this minimization step by solving an extended but unconstrained MDP via DP.\footnote{Formally, this is because a version of the Bellman optimality principle applies after dualizing the constraints, even if we need to optimize over the confidence intervals for the transitions during backward induction.} The smoothness properties of the objective of \cref{eq:primal-step} then determine the iteration complexity of the Frank-Wolfe scheme. 
Formally, we have the following.
\begin{restatable}{proposition}{lemmainnerroutine} \label{lemma:inner-routine}
    In episode $k$, fix any accuracy of $\epsilon_k > 0$. There exists an algorithm for solving \cref{eq:primal-step} such that the objective at its output $(\pi_{k}, \pt_{k})$ is $\epsilon_k$-close to the optimum of \cref{eq:primal-step}, by solving $O\left( \frac{\eta_k I S^2 A H }{\epsilon_k} \right)$ (extended) MDPs via DP.
\end{restatable}


\section{Sketch of the Regret Analysis} \label{sec:regret}

In this section, we outline the key steps in our proof of \cref{thm:optaugrlpp-regret} and defer the detailed proofs to \cref{sec:regret-omit}. We will condition our regret analysis on a \textit{success event} $G$, which we formally define in \cref{sec:regret-prelim-omit}. $G$ ensures that (a) the optimistic cost estimates in \cref{eq:opt-c} are, in fact, optimistic and (b) the true transitions are contained in the set of plausible models from \cref{eq:opt-c}, i.e.:
\begin{align*}
    \ct_k \leq c, \quad \dt_{i,k} \leq d_i \;(\forall i\in [I]),\quad p\in B_k^p,
\end{align*}
for every episode $k\in [K]$. In the following lemma, we prove that $G$ occurs with high probability.
\begin{restatable}[]{lemma}{lemgwhp} \label{lem:g-whp}
    Fix $\delta \in (0,1)$ and define the optimistic model in \cref{eq:opt-c} accordingly. Then, the success event $G$ occurs with probability at least $1-\delta$, i.e., $P[G]\geq 1-\delta$.
\end{restatable}
We proceed with the regret analysis and first split the regrets between the two phases of the algorithm:
\begin{align*}
    \Reg (K; c) &= \underbrace{\sum_{k=1}^{K'} [ V^{\pi_{k}}( c, p) - V^{\pi^*}(c,p) ]_+ }_{\text{Pre-Training}} + \underbrace{\sum_{k=K'+1}^{K} [ V^{\pi_{k}}( c, p) - V^{\pi^*}(c,p) ]_+ }_{\text{Optimistic Exploration}}, \\
    \Reg (K; d) &\leq  \underbrace{\max_{i \in [I]} \sum_{k=1}^{K'} [V^{\pi_{k}}( d_i, p) - \alpha_i]_+}_{\text{Pre-Training}} + \underbrace{\max_{i \in [I]} \sum_{k=K'+1}^{K} [V^{\pi_{k}}( d_i, p) - \alpha_i]_+}_{\text{Optimistic Exploration}}.
\end{align*}
Then, applying \cref{prop:pretraining}, we can trivially bound the objective regret during the pre-training phase by $K' H$. Since $\pibar$ is strictly feasible, there is no constraint regret during pre-training. We now focus on the regrets incurred in the optimistic exploration phase. For this, we further decompose the regrets as follows (see \cref{sec:regret-prelim-omit}):
\begin{align*}
    \Reg (K; c) &\leq K'H {+} \underbrace{\sum_{k=K'+1}^{K} [ V^{\pi_{k}}(c, p) {-} V^{\pi_{k}}( \ct_k, \pt_{k}) ]_{+}}_{\text{Estimation Error}} {+} \underbrace{\sum_{k=K'+1}^{K} [ V^{\pi_{k}}( \ct_k, \pt_{k}) {-} V^{\pi^*}(c, p) ]_{+}}_{\text{Optimization Error}}, \\
    \Reg (K; d) &\leq \underbrace{\max_{i \in [I]} \sum_{k=K'+1}^{K} [ V^{\pi_{k}}( d_i, p) {-} V^{\pi_{k}}(\dt_{i,k}, \pt_{k}) ]_{+}}_{\text{Estimation Error}} {+} \underbrace{\max_{i \in [I]} \sum_{k=K'+1}^{K} [ V^{\pi_{k}}(\dt_{i,k}, \pt_{k}) {-} \alpha_i ]_{+}}_{\text{Optimization Error}} .
\end{align*}

We have thus decomposed the regrets into
\begin{enumerate}[]
    \item[(A)] \textit{estimation errors} that are due to the estimated model, and
    \item[(B)] \textit{optimization errors} that we can analyze via the underlying optimization method.
\end{enumerate}
Conditioning on the success event $G$, we will obtain bounds sublinear in $K$ for both parts of the decomposition. Note that we cannot adapt the analysis by \citet{efroni2020exploration} to achieve this goal since it only allows for bounds on the averages of the \textit{signed} optimization errors. We proceed with bounding the estimation errors in the next section. 


\subsection{Estimation Errors (Optimistic Exploration)} \label{sec:term1}

Leveraging on-policy error bounds for optimistic exploration in MDPs (\cref{sec:onp-bounds}), we establish the desired bound on the estimation errors.
\begin{restatable}[Estimation errors]{lemma}{propestregret} \label{prop:est-regret}
    Let $(\pi_k)_{k=K'+1}^{K}$ be the sequence of policies obtained by \AlgName. Then, conditioned on $G$, we can bound the estimation errors as follows:
    \begin{align*}
        \sum_{k=K'+1}^{K} [ V^{\pi_{k}}(c, p) - V^{\pi_{k}}( \ct_k, \pt_{k}) ]_{+} \leq& \tilde{O} \left(  \sqrt{\mathcal{N}SAH^4 K} + S^2 A H^3 \right),\\
        \max_{i \in [I]} \sum_{k=K'+1}^{K} [ V^{\pi_{k}}( d_i, p) - V^{\pi_{k}}(\dt_{i,k}, \pt_{k}) ]_{+} \leq&\tilde{O} \left( \sqrt{\mathcal{N}SAH^4 K} + S^2 A H^3 \right).
    \end{align*}
\end{restatable}
We refer to \cref{sec:term1-omit} for the proof. \cref{prop:est-regret} proves that the estimation errors for both the objective and constraints can indeed be bounded by a term that is sublinear in $K$. In the next section, we provide a bound for the optimization errors.


\subsection{Optimization Errors (Optimistic Exploration)} \label{sec:term2}

Recall that by \cref{lemma:inner-routine}, our solver for the inner problem (\cref{eq:primal-step}) has the following guarantee, for every episode $k \geq K'+1$:
\begin{align*}
    \mathcal{L}_k(\pi_{k}, \pt_{k}) &\leq \min_{\substack{\pi \in \Pi \\ p' \in B_k^p}}\mathcal{L}_k(\pi, p') + \epsilon_k,
\end{align*}
where $\mathcal{L}_k(\pi, p') := V^{\pi}(\ct_k, p') + \frac{1}{2\eta_k}  \| [ \lambda_k + \eta_k ( V^{\pi}((\dt_{i,k})_{i\in[I]}, p') - \alpha ) ]_+  \|^2$ denotes the objective at episode $k$. If the true CMDP is known, i.e., no exploration is required, then \citet{xu2021iteration} proves a sublinear regret bound for the optimization error if the step sizes $\eta_k$ and accuracies $\epsilon_k$ are chosen suitably.\footnote{That is, such that $\sum_{k=K'+1}^K 1/\eta_k=  o(K)$ and $\sum_{k=K'+1}^K \epsilon_k = o(K)$, see \citet[Remark 7]{xu2021iteration}.} They obtain this result by bounding the dual variables $\lambda_k$ across the iterations $k$. In our setting, however, since the objective and constraint set of the optimization problem (\cref{eq:primal-step}) \textit{change} in every episode, we require a novel type of analysis.

\looseness=-1As a first step, we show that we can bound the optimization errors in episode $k$ by expressions that depend on the dual variables $\lambda_k$ and $\lambda_{k+1}$.
\begin{restatable}[]{lemma}{lemmafdiff} \label{lemma:fdiff} \label{lemma:gdiff}
    Conditioned on G, for each $k \in \{K'+1, \dots, K\}$, in \AlgName~we have 
    \begin{align*}
        V^{\pi_{k}}( \ct_k, \pt_{k}) - V^{\pi^*}(c, p) &\leq \epsilon_k + \frac{\| \lambda_{k} \|^2 - \| \lambda_{k+1} \|^2}{2\eta_k},\\
        V^{\pi_{k}}(\dt_{i,k}, \pt_{k}) - \alpha_i &\leq \frac{\lambda_{k+1}(i) - \lambda_{k}(i)}{\eta_k} ~~~~ (\forall i \in [I]).
    \end{align*}
\end{restatable}
To further bound the norm of the dual iterates, for each episode $k\geq K'$, we consider the $k$-th optimistic CMDP, which we define as follows:
\begin{align}
     \min_{\substack{\pi \in \Pi}} ~~~~ V^{\pi}(\ct_k, \pt_{k}) ~~~~ \text{s.t.} ~~~~  V^{\pi}(\dt_{i,k}, \pt_{k}) \leq \alpha_i ~~~~(\forall i\in[I]). \label{eq:opt-KKT}
\end{align}
Note that \cref{eq:opt-KKT} indeed is a CMDP. By \cref{prop:pretraining}, $\pibar$ is strictly feasible for \cref{eq:opt-KKT} for all $k\geq K'$, with a slack of $\geq \nu \gamma$ uniformly bounded away from zero. By strong duality \citep{paternain2019strongduality}, there exist primal-dual pairs $(\pi_k^*, \lambda_k^*)$ satisfying
\begin{equation*}
    V^{\pi_k^*}(\ct_k, \pt_{k}) = \min_{\substack{\pi \in \Pi}} \left( V^{\pi}(\ct_k, \pt_{k}) + (\lambda_k^*)^T (V^{\pi}((\dt_{i,k})_{i\in[I]}, \pt_{k}) - \alpha)\right).
\end{equation*}
We formalize this with \cref{lemma:aux-ineq} in \cref{sec:dual-iter-bound} by using the fact that we can formulate \cref{eq:opt-KKT} as a convex optimization problem using the LP formulation of CMDPs (\cref{sec:occ-notation,sec:conv-prelim}). With this, we can establish the following bound on the dual iterates.
\begin{restatable}[]{lemma}{cordualbound}\label{cor:dual-bound}
    Let $k \in \curly{K'+1, \dots, K}$ and suppose \cref{eq:opt-KKT} is strictly feasible for every $k' \in \curly{K', \dots, K}$. Let $(\pi_{k'}^*, \lambda_{k'}^*)$ be pairs of primal-optimal and dual-optimal solutions for \cref{eq:opt-KKT}. Then the iterates of \AlgName~satisfy
    \begin{align*}
        \| \lambda_{k+1}\| \leq& ~2 \sum_{t=K'}^{k} \| \lambda_t^*\| + \sum_{t=K'+1}^{k} \sqrt{2\eta_t \epsilon_t}.
    \end{align*}
\end{restatable}\vspace{-0.5em}
Having achieved a bound on the dual iterates $\lambda_{k+1}$ in terms of the dual maximizers $\lambda_{k'}^*$  (${k'} \in \curly{K', \dots, k}$), we can now aim to provide bounds for the latter. Indeed, we can leverage results from constrained convex optimization (\cref{sec:conv-prelim}) to arrive at the following lemma.
\begin{restatable}[]{lemma}{thmlstarbound} \label{thm:lstar-bound}
    Suppose Assumption \ref{ass:mild-slater} holds. Let $\nu \in (0,1)$ and choose $K'$ as in \cref{prop:pretraining}. Let $k \in \curly{K', \dots, K}$, and let $(\pi_k^*, \lambda_k^*)$ be a pair of primal-optimal and dual-optimal solutions for \cref{eq:opt-KKT}. Then, conditioned on $G$, we have
    \begin{align*}
        \|\lambda_k^*\| \leq \| \lambda_k^* \|_1 \leq \frac{H}{\nu \gamma}.
    \end{align*}
\end{restatable}\vspace{-0.5em}
Plugging \cref{cor:dual-bound} into the bounds from \cref{lemma:fdiff} and replacing the norms of the $\lambda^*_k$ using the bound from \cref{thm:lstar-bound}, we obtain sublinear optimization errors when choosing $\eta_k$, $\epsilon_k$ correctly:
\begin{restatable}[Optimization errors]{lemma}{lemmaopterrork} \label{lemma:opt-error-k}
    Suppose Assumption \ref{ass:mild-slater} holds. Let $\nu \in (0,1)$ and choose $K'$ as in \cref{prop:pretraining}. Suppose that the event $G$ occurs. When using step sizes $\eta_{K'+k} = \Theta(k^{2.5})$ and $\epsilon_{K'+k} = \Theta(1/\eta_{K'+k})$, we have
    \begin{align*}
        \sum_{k=K'+1}^{K} [ V^{\pi_{k}}( \ct_k, \pt_{k}) - V^{\pi^*}(c, p) ]_{+} \leq& \sum_{k=K'+1}^{K} \bigg( \frac{(O(\sigma k) + \sum_{t=K'+1}^{k} \sqrt{2\eta_t\epsilon_t})^2}{2\eta_k} + \epsilon_k \bigg) \leq O(\sqrt{K}), \\
        \max_{i \in [I]} \sum_{k=K'+1}^{K} [ V^{\pi_{k}}(\dt_{i,k}, \pt_{k}) - \alpha_i ]_{+} \leq& \sum_{k=K'+1}^{K} \frac{O(\sigma k) + \sum_{t=K'+1}^{k}    \sqrt{2\eta_t \epsilon_k}}{\eta_k} \leq O(\sqrt{K}),
    \end{align*}
    where $\sigma = \frac{H}{\nu \gamma}$ and in fact $O(\sigma k)$ can be replaced by $(2+2(k-K'))\sigma$.
\end{restatable}
\looseness -1\textbf{Remark:} We need to choose $\eta_k$ large enough (increasing) and $\epsilon_k$ small enough (decreasing) to ensure a sublinear error bound. At the same time, we do not want to choose $\eta_k$ larger than necessary or $\epsilon_k$ smaller than necessary for computational reasons (see \cref{lemma:inner-routine}). We refer to \cref{sec:main-res-omit} for a discussion. In the case of an exact subroutine, we can plug in $\epsilon_k = 0$ to achieve an analogous result.

According to our regret decomposition, by adding up the errors of the pre-training phase, the estimation errors in the second phase, and the optimization errors in the second phase, we can indeed deduce our main result (\cref{thm:optaugrlpp-regret}). We showed how to bound the estimation errors using the optimism paradigm (\cref{prop:est-regret}). For the analysis of the optimization errors, we had to generalize the convergence analysis of the inexact augmented Lagrangian method (\cref{lemma:opt-error-k}). We showed that if we have access to a safe baseline policy, the pre-training phase guarantees all assumptions required for this and adds a constant term to the regret (\cref{prop:pretraining}).


\section{Conclusion}

In this work, we showed how to overcome the problem of the \textit{cancellation of errors}, i.e., the oscillation of standard Lagrangian-based algorithms for CMDPs around an optimal safe policy. We leveraged the augmented Lagrangian method to design our algorithm \AlgName. Unlike the related \textsc{OptDual-CMDP} algorithm of \citet{efroni2020exploration}, this requires a subroutine that solves a non-linear optimization problem in each episode. We devised an efficient algorithm for this, avoiding projections or LP. We then provided a regret analysis that, unlike previous works, does not require the cancellation of errors to arrive at sublinear regret guarantees. This means that in contrast to existing Lagrangian-based algorithms, our algorithm is provably safe \textit{while exploring} the unknown CMDP. 

\looseness=-1This first partial answer to the open problem posed by \cite{efroni2020exploration} leads to several further questions: Can we obtain tighter bounds for the inner sub-routine and the regret, as our problem has a richer structure than the general convex optimization setup? While \AlgName~enjoys stronger regret guarantees, the proposed inner subroutine has a higher computational cost than the one in \textsc{OptDual-CMDP}, which may be possible to improve. Moreover, it remains open whether one can remove the requirement of access to a strictly feasible policy. Finally, we aim to extend our approach to the more practical function approximation setup.



\newpage

\bibliographystyle{plainnat}
\bibliography{refs}

\medskip


\newpage

\renewcommand \thepart{}
\renewcommand \partname{}

\appendix
\doparttoc
\faketableofcontents
\part{Appendix}
\parttoc


\section{Background} \label{sec:background-app}


\subsection{Review of the Augmented Lagrangian Method} \label{sec:alm}

In this section, we review the fundamentals of the augmented Lagrangian method \citep{rockafellar1976augmented, bertsekas2014constrained}, which was first introduced by \citet{hestenes1969multiplier, Powell1969AMF}. Our review is partly inspired by the review by \citet{yan2020bregman}.

Consider a closed and convex set $\Xc \subset \R^n$ and a closed convex function $f \colon \Xc \to \R$. Furthermore, let $A \in \R^{m \times n}$ and $b \in \R^m$. With this, consider the constrained problem
\begin{align}
    \min_{x \in \Xc} &~~~f(x) \label{eq:alm-opt}\\
    \text{s.t.} &~~~Ax \leq b. \nonumber
\end{align}
After initializing $x_0 \in \Xc$ and $\lambda_0 \in \R_{\geq 0}^m$, the augmented Lagrangian method performs the following updates in step $k \geq 0$:
\begin{align}
    x_{k+1} \in& \arg\min_{x \in \Xc}  \left( f(x) + \frac{1}{2\eta_k} \|[ \lambda_k + \eta_k (Ax - b)]_+\|^2 \right), \label{eq:alm1}\\
    \lambda_{k+1} =& [\lambda_k + \eta_k (Ax_{k+1} - b)]_+. \label{eq:alm2}
\end{align} 
It can easily be verified that the augmented Lagrangian method is the proximal point method applied to the Lagrangian dual. Thus, with the Lagrangian $\mathcal{L}(x,\lambda) := f(x) + \lambda^T(Ax-b)$, we have 
\begin{align*}
    (x_{k+1}, \lambda_{k+1}) \in \arg \max_{\lambda \geq 0} \min_{x \in \Xc} \left( \mathcal{L}(x,\lambda) - \frac{1}{2\eta_k} \| \lambda - \lambda_k \|^2 \right).
\end{align*}
\citet{xu2021iteration} shows a non-asymptotic convergence result for the augmented Lagrangian method, including the case of an inexact subroutine for solving \cref{eq:alm1}. Notably, the convergence result concerns the last iterate of the method. That is if $f^*$ is the optimal value of \cref{eq:alm-opt}, then $f(x_{k}) \to f^*$ as $k \to \infty$ and the convergence rate is determined by the step sizes $\eta_k$ (and, in the case of an inexact subroutine, by the accuracy parameter $\epsilon_k$). 

Alternatively, we could apply the dual projected gradient method to \cref{eq:alm-opt}. The updates would then read 
\begin{align}
    \tilde{x}_{k+1} \in& \arg\min_{x \in \Xc}  \left( f(x) + \tilde{\lambda}_k^T (Ax - b) \right), \label{eq:pgd1}\\
    \tilde{\lambda}_{k+1} =& [\tilde{\lambda}_k + \tilde{\eta}_k (A \tilde{x}_{k+1} - b)]_+. \label{eq:pgd2}
\end{align}
While \cref{eq:pgd2} coincides with \cref{eq:alm2}, the update of the primal variable differs from the one in the augmented Lagrangian method. We can view the dual projected gradient method as iterative play between a primal player $x_k$ and a dual player $\lambda_k$ in a min-max setup, where the objective is the Lagrangian $\mathcal{L}(x,\lambda)$. The primal player here plays best response, while the dual player plays online projected gradient ascent. While this is similar in spirit to applying the proximal point method to the Lagrangian dual, the known non-asymptotic convergence guarantees for this method are ergodic, i.e., they only concern convergence of the averaged iterates. Indeed, simple simulations show that this is not a weakness in the analysis but that the primal and dual iterates indeed oscillate around an optimal solution pair. This is illustrated by \citet[Chapter 8]{beck2017first}. The iterates of the dual projected gradient method may alternate between satisfying $f(x_k) > f^*$ and $Ax_k < b$ for a couple of iterations and then $f(x_k) < f^*$ and $Ax_k > b$ for a couple of iterations. While the average objective value converges to $f^*$ and the average constraint violation to $0$, this is not true for the individual iterates. Note that these oscillations are not due to estimating the problem but are present even if the optimization problem is fixed, as in the setup above. Apart from the augmented Lagrangian method, other methods, such as extra gradient or optimistic gradient descent-ascent, offer solutions to this issue.


\subsection{Review of \textsc{OptDual-CMDP}} \label{sec:optdual}

In this section, we review the related \textsc{OptDual-CMDP} algorithm of \citet{efroni2020exploration}. This model-based dual algorithm is based on the dual projected gradient method (see \cref{sec:alm}) rather than on the augmented Lagrangian method (but builds the model with the same notion of optimism).
\begin{algorithm}[H]
	\begin{algorithmic}
        \STATE{Set $\eta_k := \sqrt{\frac{\rho^2}{H^2 I K}}$ and $\lambda_1 := 0 \in \R^I$}
        \FOR{$k=1,\dots, K$}
            \STATE{Update policy:}
            \begin{align}
                \pi_{k}, \tilde{p}_{k} := \arg\min_{\substack{\pi \in \Pi \\ p' \in B_k^p}} \left( V^{\pi}(\ct_k, p') + \lambda_k^T ( V^{\pi}((\dt_{i,k})_{i\in[I]}, p') - \alpha )  \right) 
            \end{align}
            \STATE{Update dual variables:}
            \begin{align}
                \lambda_{k+1} := [ \lambda_{k} + \eta_k (  V^{\pi_{k}}((\dt_{i,k})_{i\in[I]}, \pt_{k}) - \alpha ) ]_+
            \end{align}
            \STATE{Play $\pi_{k}$, update estimates of the costs $\ct_{k+1}$, $(\dt_{i,k+1})_{i\in[I]}$ and transitions $B_{k+1}^p$ (\cref{eq:opt-c}).}
        \ENDFOR
		\caption{\textsc{OptDual-CMDP}}
	\end{algorithmic}
	\label{algo:optdual}
\end{algorithm}

Dual approaches like the algorithm above turn the CMDP into a series of linearly regularized, (extended) \textit{unconstrained} MDPs (here with objective $V^{\pi}(\ct_k, p') + \lambda_k^T (V^{\pi}((\dt_{i,k})_{i\in[I]}, p') - \alpha )$), that can be solved efficiently with DP. With the augmented Lagrangian approach, the inner problem (\cref{eq:primal-step}) has a more complicated structure. In \cref{sec:inner}, we show that the inner problem can be reformulated as a convex optimization problem and provide an efficient method based on Frank-Wolfe and DP.

For \textsc{OptDual-CMDP}, we have the following guarantee \citep[Theorem 5]{efroni2020exploration}.
\begin{restatable}[]{theorem}{thm5efroni}
    Suppose there exists a strictly feasible policy $\pi$ such that for all $i \in [I]$ we have $V^{\pi}(d_i, p) < \alpha_i $. Set
    \begin{align*}
        \rho := \frac{V^{\pi}(c,p) - V^{\pi^*}(c,p)}{\min_{i \in [I]} (\alpha_i - V^{\pi}(d_i, p))}.
    \end{align*}
    Then, for any $\delta\in (0,1)$, with probability at least $1-\delta$, \textsc{OptDual-CMDP} achieves the following regret bounds:
    \begin{align*}
        \Reg_{\pm}(K; c) &= \tilde{O} \bigg( \sqrt{S \mathcal{N} H^4 K} + \rho \sqrt{H^2 I K} + (\sqrt{\mathcal{N}} +H ) H^2 SA \bigg),\\
        \Reg_{\pm}(K;d) &= \tilde{O} \bigg( (1 + \frac{1}{\rho})(\sqrt{I S\mathcal{N} H^4 K} + (H\sqrt{ \mathcal{N}} + S) \sqrt{I}H^2S A) \bigg).
    \end{align*}
\end{restatable}
Notably, this bound only covers the \textit{weak} regret. This is because the dual projected gradient method does not allow for a non-ergodic convergence analysis, and its iterates will generally oscillate around an optimal feasible solution. It is worth mentioning that unlike LP-based approaches (\textsc{OptCMDP} and \textsc{OptCMDP-bonus} of \citet{efroni2020exploration}), the related primal-dual method \textsc{OptPrimalDual-CMDP} of \citet{efroni2020exploration} suffers from the same problem as \textsc{OptDual-CMDP}. Moreover, we remark that this is not just a hypothetical issue but that Lagrangian-based algorithms indeed suffer from the mentioned oscillations in practical applications \citep{stooke2020responsive,moskovitz2023reload}.


\subsection{CMDPs and Occupancy Measures} \label{sec:occ-notation}

We summarize the relevant quantities of the CMDP $\mathcal{M}$ as follows.
{{\renewcommand{\arraystretch}{1.6}  
\begin{table}[H]
    \fontsize{8pt}{8pt}\selectfont
    \centering
    \begin{tabular}{|>{\bfseries}l|l|}
        \hline
        Discrete state space& $\St$, with cardinality $S$\\
        Discrete action space& $\A$, with cardinality $A$\\
        \# of constraints & $I$ \\
        Initial state & $s_1$, same for each episode\\
        Time horizon &$H$, same for each episode\\
        Transition probability & $p_h(s' | s,a) = P[s_{h+1}=s' \mid s_h=s, a_h=a]$\\
        Max. \# transitions & $\N = \max_{s,a,h} |\{ s' \mid p_h(s' | s,a) > 0 \}|$ \\ \hline
        Objective cost & Random variable $C_h(s,a) \in [0,1]$,
        with $\E[C_h(s,a)]=c_h(s,a)$\\
        Constraint cost ($i\in [I]$) & Random variable $D_{i,h}(s,a)$ with $\E[D_{i,h}(s,a)]=d_{i,h}(s,a)$ \\
        Constraint encoding & $(d_i, \alpha_i )_{i\in[I]}$, with $d_{i} = (d_{i,h})_{h\in[H]}$ and $\alpha_i \in [0,H]$ \\ \hline
        Policy & $\pi = (\pi_1, \dots, \pi_H) \in \Pi$ with $\pi_h \colon \St \to \Simplex{\A}$ (non-stationary)\\ \hline
        Value function & $V^{\pi}(c,p) = \E[\sum_{h=1}^H c_{h}(s_{h}, a_{h}) \mid s_1, p, \pi]$ \\
        Constraint value function ($i\in [I]$) & $V^{\pi}(d_i,p)=\E[ \sum_{h=1}^H d_{i,h}(s_{h}, a_{h}) \mid s_1, p, \pi]$ \\\hline
    \end{tabular}
    \vspace{0.2cm}
    \caption{Summary of CMDP notation}
    \label{tab:mdp}
\end{table}
} \vspace{-0.6cm}
To view the CMDP as a convex optimization problem, we will express it via the common notion of occupancy measures \citep{borkar1988convex}.
\begin{definition}
    The state-action \textit{occupancy measure} $q^{\pi}$ of a policy $\pi$ for a CMDP $\M$ is defined as 
    \begin{align*}
        q_h^{\pi}(s,a;p) := \E \left[ \indic_{\{ s_h = s, a_h = a \}} \mid s_1, p, \pi \right] = P[ s_h=s, a_h=a \mid s_1, p, \pi ],
    \end{align*}
    for $s\in\St$, $a\in\A$, $h\in[H]$. We denote the stacked vector of these values as $q^{\pi}(p) \in \R^{SAH}$, with the element at index $(s,a,h)$ being $q_h^{\pi}(s,a;p)$.
\end{definition}
For transition probabilities $p'$, we can now define
\begin{align*}
    Q(p') := \left\{ q^{\pi}(p') \in \R^{SAH} \mid \pi \in \Pi \right\}
\end{align*}
as the \textit{state-action occupancy measure polytope}. Note that $Q(p')$ is indeed a polytope \citep{altman1999constrained,efroni2020exploration}. Moreover, for any $p'$ we have a surjective map $\pi \mapsto q^{\pi}(p')$ between $\Pi$ and $Q(p')$, for which we can explicitly compute an element in the pre-image of $q \in Q(p')$ via $\pi_h(a|s) = q_h(s,a) / (\sum_{a'} q_h(s,a'))$.

We can stack the expected costs $c_h(s,a)$ and constraint costs $d_{ih}(s,a)$ in the same way as $q_h(s,a)$ to obtain vectors $c \in \R^{SAH}$ and $d_i \in \R^{SAH}$. Note that we then have $V^{\pi}(c,p) = \sum_{h,s,a} q^{\pi}_h (s,a;p) c_h(s,a) = c^T q^{\pi}(p)$ by linearity of expectation. Similarly, for all $i\in [I]$, we have $V^{\pi}(d_i, p) = d_i^T q^{\pi}(p)$. Moreover, if we stack $D = (d_i)_{i\in[I]} \in \R^{I \times SAH}$ and $\alpha = (\alpha_i)_{i\in[I]}\in \R^I$ as
\begin{align*}
    D := \left(\begin{matrix}
    d_1^T\\
    \vdots\\
    d_I^T
    \end{matrix}\right),\hspace{1cm}
    \alpha := \left( \begin{matrix}
        \alpha_1\\
        \vdots \\
        \alpha_I
    \end{matrix} \right),
\end{align*}
we obtain $V^{\pi}(D, p) = D q^{\pi}(p) \in [0,H]^I$ for the vector of the constraint value functions. We can thus write
\begin{align*}
    \pi^* \in \arg \min_{\pi \in \Pi} ~~~~ V^{\pi}(c, p) ~~~~ \text{s.t.} ~~~~ V^{\pi}(d_i, p) \leq \alpha_i ~~~~(\forall i \in [I]) 
\end{align*}
equivalently as
\begin{align*}
    q^{\pi^*} \in \arg \min_{q^{\pi} \in Q(p)} ~~~~ c^T q^{\pi}(p) ~~~~ \text{s.t.} ~~~~ D q^{\pi}(p) \leq \alpha, \nonumber
\end{align*}
which is an LP. In particular, if we assume feasibility, then by compactness of the state-action occupancy polytope and continuity of the objective, there is an optimal solution $\pi^*$.


\subsection{Convex Optimization Preliminaries} \label{sec:conv-prelim}

We state some well-known results from constrained convex optimization that will be useful to bound the dual iterates $\lambda_k$ appearing in \cref{lemma:fdiff}. The results are standard, and we refer, for example, to the work by \citet{beck2017first}. 

Consider the (primal) optimization problem
\begin{align}
    f^* := \min ~~~ & f(x) \nonumber\\
    \text{s.t.} ~~~ & g(x) \leq 0 \label{eq:opt-P}\\
    & x \in X \nonumber
\end{align}
with the following assumptions.

\begin{assumption}[Assumption 8.41, \citet{beck2017first}] \label{ass:841}
    In \cref{eq:opt-P},
    \begin{itemize}
        \item[(a)] $X \subset \R^n $ is convex
        \item[(b)] $f\colon \R^n \to \R$ is convex
        \item[(c)] $g(\cdot) := (g_1(\cdot), \dots, g_m(\cdot)) ^T$ with $g_i \colon \R^n \to \R$ convex
        \item[(d)] \cref{eq:opt-P} has a finite optimal value $f^*$, which is attained by exactly the elements of $X^* \neq \emptyset$ 
        \item[(e)] There exists $\xbar \in X$ such that $g(\xbar) < 0$
        \item[(f)] For all $\lambda \in \R_{\geq 0}^m$, $\min_{x\in X} (f(x) + \lambda^Tg(x))$ has an optimal solution
    \end{itemize}
\end{assumption}

In this setup, we define the \textit{dual objective} as
\begin{align*}
    q(\lambda) := \min_{x \in X} \left( f(x) + \lambda^T g(x) \right),
\end{align*}
where $\mathcal{L} \colon \R^n \times \R^m \to \R$, $\mathcal{L}(x;\lambda) := f(x) + \lambda^T g(x)$ is the \textit{Lagrangian} of the problem in \cref{eq:opt-P}. The \textit{dual problem} is then defined as 
\begin{align*}
    q^* := \max ~~~ & q(\lambda)\\
    \text{s.t.} ~~~ & \lambda \geq 0.
\end{align*}
In this setup, we have the following results connecting the primal and the dual problem.

\begin{restatable}[Theorem A.1, \citet{beck2017first}]{theorem}{thmduality} \label{thm:duality}
    Under Assumption \ref{ass:841}, strong duality holds in the following sense: We have 
    \begin{align*}
        f^* = q^*
    \end{align*}
    and the optimal solution of the dual problem is attained, with the set of optimal solutions $\Lambda^* \neq \emptyset$.
\end{restatable}

\begin{proof}
     Proposition 6.4.4 of \citet{bertsekas2003convex} proves the more general Theorem A.1 of \citet{beck2017first}. We remark that if we assume affine constraints $g$ and $X$ being a polytope, then we can drop assumption (e) \citep[Theorem A.1]{beck2017first}. 
\end{proof}

\begin{restatable}[]{theorem}{thmconvauxineq}\label{thm:conv-aux-ineq}
    Suppose Assumption \ref{ass:841} holds. Let $x^*\in X^*$, $\lambda^* \in \Lambda^*$ and $x \in X$. Then 
    \begin{align*}
        f(x) - f(x^*) + (\lambda^*)^T g(x) \geq 0.
    \end{align*}
\end{restatable}

\begin{proof}
    We have 
    \begin{align*}
        f(x) =& f(x) + (\lambda^*)^T g(x) - (\lambda^*)^T g(x)&\\
        \geq& q(\lambda^*)- (\lambda^*)^T g(x) & \text{(definition of } q(\cdot)\text{)}\\
        =& f(x^*) - (\lambda^*)^T g(x),& \text{(since by \cref{thm:duality}, } q^*=f^*\text{)}
    \end{align*}
    and rearranging this proves the claim. Again, we can drop assumption (e) if we consider affine constraints $g$ and a polytope $X$.
\end{proof}

\begin{restatable}[]{theorem}{thmconvdual} \label{thm:conv-dual}
    Under Assumption \ref{ass:841}, for all $\lambda^* \in \Lambda^*$ and $\xbar$ as in (e), we have 
    \begin{align*}
        \| \lambda^* \| \leq \|\lambda^* \|_1 \leq \frac{f(\xbar)  - f^*}{\min_{i\in [m]} (-g_i(\xbar))}.
    \end{align*}
\end{restatable}

\begin{proof}
    The first relation holds since $\lambda^* \geq 0$. We show the second relation as follows (cf. \citet[Theorem 8.42]{beck2017first}). We have 
    \begin{align*}
        f(x^*) =& q(\lambda^*) &\text{(\cref{thm:duality})}\\
        \leq& f(\xbar) + (\lambda^*)^T g(\xbar) &\text{(definition of } q(\cdot)\text{)}\\
        \leq& f(\xbar) + \|\lambda^*\|_1 \max_{i\in[m]} g_i(\xbar) &\text{(since } \lambda^* \geq 0\text{)}\\
        =& f(\xbar) - \|\lambda^*\|_1 \min_{i\in[m]} (-g_i(\xbar)) & 
    \end{align*}
    and rearranging this proves the claim. We remark that this theorem needs assumption (e), even in the affine case. 
\end{proof}


\section{Solving the Inner Optimization Problem} \label{sec:inner}

There are numerous works leveraging Frank-Wolfe schemes for RL (including planning) with convex objectives, most commonly in the context of pure/active exploration \citep{hazan2019provably, tarbouriech2019active, tarbouriech2020active, mutny2023active}.

To the best of our knowledge, only \citet[Appendix A.3]{tarbouriech2019active} remark that a combination of Frank-Wolfe UCB and planning in extended MDPs would be possible with a convex objective and plausible transitions. \citet{tarbouriech2020active} follow this approach but solve an extended LP rather than an extended MDP in each Frank-Wolfe iteration. Therefore, we make the former idea explicit and show how we can devise an efficient algorithm for our case.


\subsection{Derivation of the Policy Update via Frank-Wolfe} \label{sec:fw-inner}

In the following, we provide an efficient algorithm for solving the inner optimization problem in \cref{eq:primal-step} based on the extended LP formulation of CMDPs, Frank-Wolfe, and DP. Recall the first update in \AlgName while rewriting the value functions in terms of occupancy measures (see \cref{sec:occ-notation}):
\begin{align}
    &\pi_{k}, \tilde{p}_{k} := \arg\min_{\substack{\pi \in \Pi \\ p' \in B_k^p}} \left( \ct_k^T q^{\pi}(p') + \frac{1}{2\eta_k} \left \| [ \lambda_k + \eta_k ( \Dt_k q^{\pi}(p') - \alpha ) ]_+ \right \|^2 \right), \label{eq:step1}
\end{align}
where $q^{\pi}(p') \in \R^{SAH}$, $\Dt_k \in \R^{I \times SAH}$ and $\ct_k \in \R^{SAH}$ are defined as in \cref{sec:occ-notation}.

We first use the extended LP trick \citep{rosenberg2019online,efroni2020exploration} to switch to a convex optimization problem in the state-action-state occupancy measure. That is, in the problem above, substitute
\begin{align}
    z_h(s,a,s') := z_h^{\pi}(s,a,s'; p') :=& p'_h(s' | s,a) q_h^{\pi}(s,a;p')  \label{eq:subst}
\end{align}
and note that
\begin{align}
    q^{\pi}_h(s,a;p') = \sum_{s'} z_h^{\pi}(s,a,s';p'). \label{eq:retrieve-q}
\end{align}
Stacked across $(s,a,h) \in \St \times \A \times [H]$, this simply reads
\begin{align}
    q^{\pi}(p') = \sum_{s'} z^{\pi}(s';p') \in \R^{SAH}, \nonumber
\end{align}
for all $s' \in \St$. The objective of \cref{eq:step1} then reads
\begin{align}
    f(z) := \sum_{s'} \ct_k^T z(s') + \frac{1}{2\eta_k}  \| [ \lambda_k + \eta_k ( \sum_{s'} \Dt_k z(s') - \alpha ) ]_+ \|^2, \nonumber
\end{align}
and we need to minimize it over the set $Z \subset \R^{S^2AH}$ which is given by the constraints
\begin{align}
    \begin{cases} 
    \sum_{a,s'} z_h(s,a,s') = \sum_{s',a'} z_{h-1}(s',a',s) \tab& (\forall h > 1, s )\nonumber\\[7pt]
    \sum_{a,s'} z_1(s,a,s') = \mu(s) \tab& (\forall s)\nonumber\\[7pt]
    z_h(s,a,s') \geq 0 \tab& (\forall s,a,s',h)\nonumber\\[7pt]
    z_h(s,a,s') - (\pb^{k-1}_h(s' | s, a) + \beta_{k,h}^{p}(s,a,s'))\sum_{s''} z_h(s,a,s'') \leq 0\tab& (\forall s,a,s',h)\nonumber\\[7pt]
    - z_h(s,a,s') + (\pb^{k-1}_h(s' | s, a) - \beta_{k,h}^{p}(s,a,s'))\sum_{s''} z_h(s,a,s'') \leq 0 \tab& (\forall s,a,s',h), \nonumber
    \end{cases}
\end{align}
where $\mu(s) = 1$ if $s=s_1$ and $0$ otherwise. Note that $Z$ is a bounded polytope and thus compact and convex. We can thus equivalently solve 
\begin{align}
    \min_{z \in Z} ~f(z), \label{eq:convex-fw}
\end{align}
which is a standard convex optimization problem. Note that, due to the nonlinear objective, even with the convex formulation in \cref{eq:convex-fw}, \cref{eq:primal-step} cannot be written as an LP. It cannot be rewritten as a standard convex RL problem \citep{zahavy2021reward,geist2021concave} either since the transition probabilities are part of the optimization, and we thus had to switch to the extended convex program in the state-action-state occupancy measure $z$. Note that for efficiency reasons, we aim to avoid LP (which is needed for solving CMDPs due to the lack of the Bellman optimality principle) and projections onto the high-dimensional constraint set $Z$. Instead, we will make use of DP and gradient-based methods. 

If we can approximately solve \cref{eq:convex-fw}, we can later retrieve transitions and policy via
\begin{align}
    \pt_{k,h}(s' | s, a) &:= \frac{z_h(s,a,s')}{\sum_{s''} z_h(s,a,s'')}, \label{eq:retrieve-p}\\
    \pi_{k,h}(a | s) &:= \frac{\sum_{s'} z_h(s,a,s')}{\sum_{a',s'} z_h(s,a',s')}. \label{eq:retrieve-pi}
\end{align}
We use Frank-Wolfe to solve \cref{eq:convex-fw}. Frank-Wolfe is a well-known iterative method for constrained nonlinear optimization, which minimizes a smooth convex function over a convex domain and avoids projections. To apply it to \cref{eq:convex-fw}, we need a linear minimization oracle (LMO) that, in step $t$ (of episode $k$), solves
\begin{align}
    \min_{g \in Z} ~g^T \nabla f(z^t)
\end{align}
and then, after finding a minimizer $g$, updates 
\begin{align}
    z^{t+1} := (1-\gamma_t)z^{t} + \gamma_t g, \nonumber
\end{align}
where $\gamma_t := 2/(t+2)$. It is well known that Frank-Wolfe converges with a rate of $O(1/T)$ for smooth objectives. This is given here, but the smoothness parameter depends on $\eta_k$. We refer to \citet{abernethy2017frank,jaggi2013revisiting,frank1956algorithm} for the relevant background.

It remains to show that we can provide an efficient LMO that uses DP instead of LP and avoids projections onto $Z$. Let $ (\nabla f(z))(s, a, s', h) := \frac{\partial f}{\partial z_h(s,a,s')}(z)$ be the gradient of $f$ with respect to $z$ at index $(s,a,s',h)$. Then the $s'$-th component of the gradient of $f$ with respect to $z$ is
\begin{align}
    (\nabla f(z))(\cdot, \cdot, s', \cdot) = \ct_k + \Dt_k^T [\lambda_k + \eta_k ( \Dt_k \sum_{s''} z(s'') - \alpha )]_+ \in \R^{SAH}, \label{eq:nabla}
\end{align}
which we can compute explicitly and efficiently. We note that this gradient does not depend on $s'$, and thus the gradient of $f$ with respect to the whole vector $z \in \R^{S^2 A H}$ is simply the vector above, repeatedly stacked $S$ times.

We now show how switching back to the optimization over $\Pi \times B_k^p$ allows for an efficient LMO via DP. For $g \in Z$, there are $\pi$, $p'$ such that
\begin{align}
    g_h(s,a,s') = g_h^{\pi}(s,a,s'; p') = p_h(s' | a, s) q_h^{\pi}(s,a;p), \nonumber
\end{align}
via \cref{eq:retrieve-pi,eq:retrieve-p}. The LMO then needs to minimize
\begin{align}
    g^T \nabla f(z^t) =& \sum_{s,a,s',h} g_h(s,a,s') (\nabla f(z^t))(s,a,1,h)\nonumber\\
    =& \sum_{s,a,s',h}  p'_h(s' | a, s) q_h^{\pi}(s,a;p') (\nabla f(z^t))(s,a,1,h)\nonumber\\
    =& \sum_{s,a,h} q_h^{\pi}(s,a;p') (\nabla f(z^t))(s,a,1,h) \label{eq:lmo-policy}
\end{align}
over $\pi \in \Pi$ and $p' \in B_k^p$. This corresponds to solving the extended MDP $\M^+ :=\{(M=(\St, \A, r^+, p^+)) \mid \forall s,a,h \colon r^+_h(s,a) := \nabla f(z^t)(s,a,1,h),~ p_h^+(\cdot | s,a) \in B_{k,h}^p(s,a) \}$. We can do so via backward induction (i.e., DP) that optimizes
\begin{align}
    Q_h^k(s,a) := r_h^+(s,a) + \min_{p'(\cdot | s,a) \in B_{k,h}^p(s,a)} \sum_{s'} p'(s' | s,a ) \min_{a'} Q_{h+1}^k (s',a') \label{eq:bw-ind}
\end{align}
and starts with $Q_{H+1}^k(s,a)=0$. We can retrieve the transitions and the policy by storing the minimizers in each step of the DP. To compute the solution $g$ of the LMO from $\pi$, $p'$, one can then use a simple and efficient DP scheme to compute $q_h^{\pi}(p')$, which is explained in \cref{sec:pseudo-inner}. We can use this, in turn, to retrieve $g$ using the substitution in \cref{eq:subst}. Then, we perform the second step of Frank-Wolfe to get a convex combination of $g$ and $z^t$, which concludes the Frank-Wolfe iteration.

We can already see that the computational complexity of the $k$-th step of \AlgName~is larger than the one of \textsc{OptDual-CMDP} by a factor of $O(1 / \epsilon_k)$ (and a dependency on $\eta_k$) because we need this many Frank-Wolfe iterations, which is the price we pay for the stronger regret bound. We provide a complete analysis of the iteration complexity in \cref{sec:iter-compl}.


\subsection{Pseudocode for the Inner Problem} \label{sec:pseudo-inner}

More formally, the algorithm for solving the inner problem (\cref{eq:primal-step}) reads as follows.
\begin{algorithm}[H]
	\textbf{Input:} Current estimates $\ct_k$, $\Dt_k$, $B_k^p$, and $\lambda_k$ \\
	\textbf{Output:} Next policy $\pi_{k}$ and $\pt_{k}$ according to \cref{eq:step1}
	
	\begin{algorithmic}
        \STATE{Set $T = \frac{2 \eta_k I S^2 A H }{\epsilon_k} $, $\gamma_t = 2/(2+t)$}
        \STATE{Initialize $z^0 \in Z$ arbitrarily}
        \FOR{$t = 0, \dots, T-1$}
            \STATE{Compute $\nabla f(z^t)$ according to \cref{eq:nabla}}
            \STATE{Minimize \cref{eq:lmo-policy} via DP algorithm from \cref{eq:bw-ind} to find $\pi^t$, $p^t$}
            \STATE{Retrieve minimizer $g^t$ from $\pi^t$, $p^t$} via DP algorithm from \cref{eq:jin-recurstion} and \cref{eq:subst} \label{line:jin-rec}
            \STATE{Set $z^{t+1} := \gamma_t g^t + (1-\gamma_t) z^t$}
        \ENDFOR
        \STATE{Construct $\pi_{k}$, $\pt_{k}$ from $z^T$ via \cref{eq:retrieve-pi,eq:retrieve-p}}
        \RETURN{$\pi_{k}$, $\pt_{k}$}
		\caption{\textsc{InnerOpt-FW}}
	\end{algorithmic}
	\label{algo:step1}
	
\end{algorithm}

As a final step, we now describe how to retrieve $q := q^{\pi}(p')$ from $\pi$ and $p'$ in Line \ref{line:jin-rec} \citep{jin2019learning}. For $s\in\St$, set $q_h(s) := P[s_h=s | s_1; p', \pi] = \sum_{a' \in \A} q_h(s,a')$. Given a policy $\pi$ and transition probabilities $p'$, we need to compute 
\begin{align}
    q_h(s, a) = q_h(s) \cdot \pi_h(a | s), \label{eq:decomp-occ}
\end{align}
for every $h \in [H]$, $s\in\St $ and $a\in \A$. This is easily achieved via DP. Indeed, we have 
\begin{align*}
    q_1(s) = \begin{cases}
        1 \tab (s=s_1)\\
        0 \tab (\text{else}),
    \end{cases}
\end{align*}
and for $h > 1$
\begin{align}
    q_h(s) =& P[s_h=s | s_1; p', \pi] \nonumber\\
    =& \sum_{s' \in \St } P[s_h=s | s_{h-1} = s'; s_1, p', \pi] P[s_{h-1} = s' | s_1; p',\pi]\nonumber\\
    =& \sum_{s' \in \St } \sum_{a \in \A} \pi_{h-1}(a | s') p_{h-1}(s | s', a) P[s_{h-1} = s' | s_1,; p',\pi] \nonumber\\
    =& \sum_{s' \in \St } \sum_{a \in \A} \pi_{h-1}(a | s') p_{h-1}(s | s', a) q_{h-1}(s'), \label{eq:jin-recurstion}
\end{align}
which together with \cref{eq:decomp-occ} enables us to retrieve the state-action occupancy measure. Clearly, we can perform the above DP scheme efficiently.


\section{On-Policy Error Bounds} \label{sec:onp-bounds}

We consider arbitrary polices $(\pi_k)_{k\in [K]}$. We suppose that in the CMDP $\M$, the agent plays $\pi_k$ in episode $k \in [K]$ and uses it to update the optimistic model according to \cref{eq:opt-c}, with some fixed $\delta = 3 \delta' > 0$.

We first establish \cref{lem:l36-efroni,lem:37efr,lem:onp-error}, which will allow us to bound the estimation errors (\cref{prop:est-regret}). For a definition of the occupancy measure $q^{\pi}(s,a;p)$, see \cref{sec:occ-notation}. We write $\lesssim$ for an inequality up to polylogarithmic factors.

Note that in the following two lemmas, the exponent of $H$ differs from the one in the referenced proofs. This is because the referenced works consider the case of stationary transition probabilities, whereas we consider non-stationary dynamics. See \citet[Lemmas 18, 19]{shani2020optimistic}.

\begin{lemma}[Lemma 36, \citet{efroni2020exploration}] \label{lem:l36-efroni}
    Suppose for all $s$, $a$, $h$, $k \in [K]$ we have 
    \begin{align*}
        n_{h}^{k-1}(s,a) &> \frac{1}{2} \sum_{j<k} q_h^{\pi_j}(s,a;p) - H \log\left( \frac{SAH}{\delta'}\right).
    \end{align*}
    Then for all $K' \leq K$
    \begin{align*}
       \sum_{k'=1}^{K'} \sum_{h=1}^{H} \E \left[ \frac{1}{\sqrt{n^{k'-1}_h(s_h^{k'},a_h^{k'})}} \mid \mathcal{F}_{k'-1} \right] \leq \tilde{O}( \sqrt{SAH^2K'} + SAH ),
    \end{align*}
    where $\mathcal{F}_{k'-1}$ is the $\sigma$-algebra induced by all random variables up to and including episode $k'-1$.
\end{lemma}

\begin{proof}
    We refer to \citet[Lemma 38]{efroni2019tight} for a proof of the statement.
\end{proof}

\begin{lemma}[Lemma 37, \citet{efroni2020exploration}] \label{lem:37efr}
    Suppose for all $s$, $a$, $h$, $k \in [K]$ we have 
    \begin{align*}
        n_{h}^{k-1}(s,a) &> \frac{1}{2} \sum_{j<k} q_h^{\pi_j}(s,a;p) - H \log\left( \frac{SAH}{\delta'}\right).
    \end{align*}
    Then for all $K' \leq K$
    \begin{align*}
       \sum_{k'=1}^{K'} \sum_{h=1}^{H} \E \left[ \frac{1}{n^{k'-1}_h(s_h^{k'},a_h^{k'})} \mid \mathcal{F}_{k'-1} \right] \leq \tilde{O}( SAH^2 ),
    \end{align*}
    where $\mathcal{F}_{k'-1}$ is the $\sigma$-algebra induced by all random variables up to and including episode $k'-1$.
\end{lemma}

\begin{proof}
    We refer to \citet[Lemma 13]{zanette2019tighter} for a proof of the statement. 
\end{proof}

The following lemma provides an on-policy error bound based on the value difference lemma. Together with the preliminaries from \cref{sec:regret-prelim-omit}, it allows us to bound the estimation error.

\begin{restatable}[On-policy errors; Lemma 29, \citet{efroni2020exploration}]{lemma}{lem:onp-error} \label{lem:onp-error}
    Consider an MDP with transition dynamics $p$ and arbitrary estimated transition dynamics $\hat{p}_k$ (for $k\in[K]$, each forming a conditional probability measure). Consider policy iterates $(\pi_k)_{k\in[K]}$ and suppose $\pi_k$ is played in episode $k \in [K]$ and used to update the counters. Let $l_h(s,a)$, $\tilde{l}_{k,h}(s,a)$ be the cost and corresponding optimistic cost with $l=c$ or $l = d_i$ as discussed in \cref{eq:opt-c}. For a policy $\pi$, let $V^{\pi}_h(s;l,p)$, $V^{\pi}_h(s; \tilde{l}_k, \hat{p}_k)$ be the values of $\pi$ according to the true resp. estimated model. Assume that for all $s,a,h,k$, we have 
    \begin{align*}
        n_{h}^{k-1}(s,a) &> \frac{1}{2} \sum_{j<k} q_h^{\pi_j}(s,a;p) - H \log\left( \frac{SAH}{\delta'}\right), \tag{a}\\
        |\tilde{l}_{k,h}(s,a) - l_h(s,a)| &\leq \tilde{O} \left( \frac{1}{\sqrt{n_{h}^{k-1}(s,a) \vee 1}} \right), \tag{b}\\
        |\hat{p}_{k,h}(s' | s,a) - p_h(s' | s,a) | &\leq \tilde{O} \left( \sqrt{\frac{p_h(s'| s,a)L_{\delta}^p}{n_{h}^{k-1}(s,a) \vee 1}} + \frac{L_{\delta}^p}{n_{h}^{k-1}(s,a) \vee 1} \right). \tag{c}
    \end{align*}
    Then we have 
    \begin{align}
        \sum_{k=1}^K | V^{\pi_k}( l,p) - V^{\pi_k}(\tilde{l}_k, \hat{p}_k) | \leq \tilde{O} \left( \sqrt{\mathcal{N}SAH^4 K} + S^2 A H^3 \right).
    \end{align}
\end{restatable}

\begin{proof}
    From the value difference lemma \citep[Lemma E.15]{dann2017unifying}, we get 
    \begin{align*}
        &\sum_{k=1}^K | V^{\pi_k}(l,p) - V^{\pi_k}(\tilde{l}_{k}, \hat{p}_{k}) | \\
        \leq& \sum_{k=1}^K \sum_{h=1}^H \E[ |l_h(s_h,a_h) - \tilde{l}_{k,h}(s_h,a_h) | \mid s_1, p, \pi_k ] \\
        + & \sum_{k=1}^K\sum_{h=1}^H \E\bigg[\sum_{s'} | p_h(s' | s_h, a_h) - \hat{p}_{k}(s' | s_h, a_h) |  V_{h+1}^{\pi_k} (s'; \tilde{l}_k, \hat{p}_{k}) \mid s_1, p, \pi_k \bigg].
    \end{align*}
    We can bound the first of the two terms as follows. 
    \begin{align*}
        &\sum_{k=1}^K \sum_{h=1}^H \E[ |l_h(s_h,a_h) - \tilde{l}_{k,h}(s_h,a_h) | \mid s_1, p, \pi_k ]\\
        \overset{(b)}{\lesssim}& \sum_{k=1}^K \sum_{h=1}^H \E\bigg[ \frac{1}{\sqrt{n^{k-1}_h(s_h,a_h) \vee 1}} \mid s_1, p, \pi_k \bigg]\\
        =& \sum_{k=1}^K \sum_{h=1}^H \E\bigg[ \frac{1}{\sqrt{n^{k-1}_h(s_h^{k},a_h^{k}) \vee 1}} \mid \mathcal{F}_{k-1} \bigg]\\
        \overset{\text{\cref{lem:l36-efroni}}}{\leq}& \tilde{O} \left( \sqrt{ S A H^2 K } + SAH \right),
    \end{align*}
    where the second to last inequality holds since $\pi_k$ is played in episode $k$, and Lemma 8 applies due to assertion (a).
    For the second term, we first note that $|V_{h+1}^{\pi_k} (s'; \tilde{l}_k, \hat{p}_{k})| \lesssim H$ since $|\tilde{l}_{k,h}(s,a) | \leq l_h(s,a) + \frac{1}{\sqrt{n^{k-1}_h(s,a) \vee 1}}$ by (b). Hence
    \begin{align*}
        &\sum_{k=1}^K \sum_{h=1}^H \E\bigg[\sum_{s'} | p_h(s' | s_h, a_h) - \hat{p}_{k}(s' | s_h, a_h)| V_{h+1}^{\pi_k} (s'; \tilde{l}_k, \hat{p}_{k}) \mid s_1, p, \pi_k \bigg]\\
        \overset{(c)}{\lesssim}& H \sum_{k=1}^K \sum_{h=1}^H \E\bigg[ \frac{1}{\sqrt{n^{k-1}_h(s_h,a_h) \vee 1}} \big(\sum_{s'} \sqrt{p_h(s' | s_h,a_h)}\big) + \frac{S}{n^{k-1}_h(s_h,a_h) \vee 1} \mid s_1, p, \pi_k \bigg]\\
        \overset{\text{Jensen}}{\leq}& H \sum_{k=1}^K \sum_{h=1}^H \E\bigg[ \frac{1}{\sqrt{n^{k-1}_h(s_h,a_h) \vee 1}} \sqrt{\mathcal{N}} \sqrt{\sum_{s'} p_h(s' | s_h,a_h)} + \frac{S}{n^{k-1}_h(s_h,a_h) \vee 1} \mid s_1, p, \pi_k \bigg]\\
        =& H \sum_{k=1}^K \sum_{h=1}^H \E\bigg[ \frac{\sqrt{\mathcal{N}}}{\sqrt{n^{k-1}_h(s_h^k,a_h^k) \vee 1}} + \frac{S}{n^{k-1}_h(s_h^k,a_h^k) \vee 1} \mid \mathcal{F}_{k-1} \bigg]\\
        \lesssim& H\sqrt{\mathcal{N}} \cdot \sqrt{SAH^2K} + H\sqrt{\mathcal{N}} \cdot SAH + H S \cdot SAH^2 \\
        \lesssim& \sqrt{\mathcal{N}SAH^4 K} + S^2 A H^3,
    \end{align*}
    where the second to last relation holds due to \cref{lem:l36-efroni} and \cref{lem:37efr} (which in turn apply due to assertion (a)), and the one before holds since $\pi_k$ is played in episode $k$.
\end{proof}

Next, we prove \cref{lem:onp-error-last}, which is a variation of \cref{lem:onp-error} and will allow us to establish the bound on the pre-training duration in \cref{prop:pretraining}. For this, we first need to establish \cref{lem:sqrt-sum,lem:sum}, which are a consequence of \cref{lem:l36-efroni,lem:37efr}, respectively.

\begin{lemma} \label{lem:sqrt-sum}
    Let $K' \in [K]$. Assume that for all $s$, $a$, $h$, $k \in [K]$ we have 
    \begin{align*}
        n_{h}^{k-1}(s,a) &> \frac{1}{2} \sum_{j<k} q_h^{\pi_j}(s,a;p) - H \log\left( \frac{SAH}{\delta'}\right).
    \end{align*}
    Let $\pi$ be any fixed policy and suppose that in episode $k\in [K]$, policy $\pi_k$ is played, with $\pi_{k'} = \pi$ for all $k' \in [K']$. Then, for all $k \geq K'$, we have 
    \begin{align*}
        \sum_{h=1}^H \E \left[ \frac{1}{\sqrt{n^{k-1}_h(s_h,a_h) \vee 1}} \mid s_1, \pi, p \right] \leq \tilde{O} \left( \sqrt{ S A H^2 } (K')^{-1/2} + SAH (K')^{-1} \right).
    \end{align*}
\end{lemma}

\begin{proof}
    Note that for any realization of all random variables up to and including episode $K$, we have (for $k \in \curly{0, \dots, K-1}$) 
    \begin{align*}
        n^{k+1}_h(s,a) \geq n^{k}_h(s,a),
    \end{align*}
    since the counters can only increase across episodes, and thus, for all $k' \leq K' < k$, we have
    \begin{align*}
        \frac{1}{\sqrt{n^{k-1}_h(s_h,a_h) \vee 1}} \leq \frac{1}{\sqrt{n^{k'-1}_h(s_h,a_h) \vee 1}}.
    \end{align*}
    For $k \geq K'$, we thus find 
    \begin{align*}
        \sum_{h=1}^{H} \E \bigg[ \frac{1}{\sqrt{n^{k-1}_h(s_h^{k},a_h^{k})}} \mid s_1, \pi, p  \bigg] \leq & \frac{1}{K'} \sum_{k'=1}^{K'} \sum_{h=1}^{H} \E \bigg[ \frac{1}{\sqrt{n^{k'-1}_h(s_h,a_h)}} \mid s_1, \pi, p  \bigg]\\
        =& \frac{1}{K'} \sum_{k'=1}^{K'} \sum_{h=1}^{H} \E \bigg[ \frac{1}{\sqrt{n^{k'-1}_h(s_h^{k'},a_h^{k'})}} \mid \mathcal{F}_{k'-1} \bigg]\\
        \leq& \frac{1}{K'} \tilde{O}( \sqrt{SAH^2 K'} + SAH ),
    \end{align*}
    where $\mathcal{F}_{k'-1}$ is the $\sigma$-algebra induced by all random variables up to and including episode $k'-1$ and where the first relation holds by monotonicity of the counters, the second relation holds since $\pi$ is played in the $k'$-th episode. The final relation holds by \cref{lem:l36-efroni}.
\end{proof}

\begin{lemma} \label{lem:sum}
    Let $K' \in [K]$. Assume that for all $s$, $a$, $h$, $k \in [K]$ we have 
    \begin{align*}
        n_{h}^{k-1}(s,a) > \frac{1}{2} \sum_{j<k} q_h^{\pi_j}(s,a;p) - H \log\left( \frac{SAH}{\delta'}\right).
    \end{align*}
    Let $\pi$ be any fixed policy and suppose that in episode $k\in [K]$, policy $\pi_k$ is played, with $\pi_{k'} = \pi$ for all $k' \in [K']$. Then, for all $k \geq K'$, we have 
    \begin{align*}
        \sum_{h=1}^H \E \left[ \frac{1}{n^{k-1}_h(s_h,a_h) \vee 1} \mid s_1, \pi, p \right] \leq \tilde{O} \left( SAH^2 (K')^{-1}\right).
    \end{align*}
\end{lemma}

\begin{proof}
    Note that for any realization of all random variables up to and including episode $K$, we have (for $k \in \curly{0, \dots, K-1}$) 
    \begin{align*}
        n^{k+1}_h(s,a) \geq n^{k}_h(s,a),
    \end{align*}
    since the counters can only increase across episodes, and thus, for all $k' \leq K' < k$, we have
    \begin{align*}
        \frac{1}{n^{k-1}_h(s_h,a_h) \vee 1} \leq \frac{1}{n^{k'-1}_h(s_h,a_h) \vee 1}.
    \end{align*}
    For $k \geq K'$, we thus find 
    \begin{align*}
        \sum_{h=1}^{H} \E \bigg[ \frac{1}{n^{k-1}_h(s_h,a_h)} \mid s_1, \pi, p  \bigg] \leq & \frac{1}{K'} \sum_{k'=1}^{K'} \sum_{h=1}^{H} \E \bigg[ \frac{1}{n^{k'-1}_h(s_h,a_h)} \mid s_1, \pi, p  \bigg]\\
        =& \frac{1}{K'} \sum_{k'=1}^{K'} \sum_{h=1}^{H} \E \bigg[ \frac{1}{n^{k'-1}_h(s_h^{k'},a_h^{k'})} \mid \mathcal{F}_{k'-1} \bigg]\\
        \leq& \frac{1}{K'} \tilde{O}( SAH^2 ),
    \end{align*}
    where $\mathcal{F}_{k'-1}$ is the $\sigma$-algebra induced by all random variables up to and including episode $k'-1$ and where the first relation holds by monotonicity of the counters, the second relation holds since $\pi$ is played in the $k'$-th episode. The final relation holds by \cref{lem:37efr}.
\end{proof}

We are now ready to prove the needed variation of \cref{lem:onp-error}.

\begin{lemma} \label{lem:onp-error-last} (Last-iterate fixed policy errors)
    Consider an MDP with transition dynamics $p$ and arbitrary estimated transition dynamics $\hat{p}_k$ (for $k\in[K]$, each forming a probability measure). Consider policy iterates $(\pi_k)_{k\in [K]}$ and suppose $\pi_k$ is played in episode $k \in [K]$ and used to update the counters. Let $l_h(s,a)$, $\tilde{l}_{k,h}(s,a)$ be cost and corresponding optimistic cost with $l=c$ or $l = d_i$ as discussed in \cref{eq:opt-c}. Consider $K' \in [K]$ and a policy $\pi$. Suppose that for all $k' \in [K']$, $\pi_{k'} = \pi$. Let $V^{\pi}_h(s;l,p)$, $V^{\pi}_h(s; \tilde{l}_k, \hat{p}_k)$ be the values of $\pi$ according to the true and estimated model, respectively. Assume that for all $s,a,h,k$ we have 
    \begin{align*}
        n_{h}^{k-1}(s,a) &> \frac{1}{2} \sum_{j<k} q_h^{\pi_j}(s,a;p) - H \log\left( \frac{SAH}{\delta'}\right), \tag{a}\\
        |\tilde{l}_{k,h}(s,a) - l_h(s,a)| &\leq \tilde{O} \left( \frac{1}{\sqrt{n_{h}^{k-1}(s,a) \vee 1}} \right), \tag{b}\\
        |\hat{p}_{k,h}(s' | s,a) - p_h(s' | s,a) | &\leq \tilde{O} \left( \sqrt{\frac{p_h(s'| s,a)L_{\delta}^p}{n_{h}^{k-1}(s,a) \vee 1}} + \frac{L_{\delta}^p}{n_{h}^{k-1}(s,a) \vee 1} \right). \tag{c}
    \end{align*}
    Then, for all $k \geq K'$, we have 
    \begin{align*}
        | V^{\pi}( l,p) - V^{\pi}( \tilde{l}_{k}, \hat{p}_{k}) | \leq \tilde{O} \left( \sqrt{\mathcal{N}SAH^4} (K')^{-1/2} + S^2 A H^3 (K')^{-1} \right).
    \end{align*}
\end{lemma}

In other words, if the agent plays $\pi$ for the first $K'$ episodes, then for all future episodes, the value differences for $\pi$ are upper bounded by the value on the RHS.

\begin{proof}
    From the value difference lemma \citep[Lemma E.15]{dann2017unifying} we get 
    \begin{align*}
        &| V^{\pi}( l,p) - V^{\pi}( \tilde{l}_{k}, \hat{p}_{k}) | \\
        \leq& \sum_{h=1}^H \E[ |l_h(s_h,a_h) - \tilde{l}_h(s_h,a_h) | \mid s_1, p, \pi ] \\
        + &\sum_{h=1}^H \E\bigg[\sum_{s'} | p_h(s' | s_h, a_h) - \hat{p}_{k}(s' | s_h, a_h) | V_{h+1}^{\pi} (s'; \tilde{l}_k, \hat{p}_{k}) \mid s_1, p, \pi \bigg].
    \end{align*}
    We can bound the first of the two terms as follows. 
    \begin{align*}
        &\sum_{h=1}^H \E[ |l_h(s_h,a_h) - \tilde{l}_h(s_h,a_h) | \mid s_1, p, \pi ]\\
        \overset{(b)}{\lesssim}& \sum_{h=1}^H \E\bigg[ \frac{1}{\sqrt{n^{k-1}_h(s_h,a_h) \vee 1}} \mid s_1, p, \pi \bigg]\\
        \overset{\text{\cref{lem:sqrt-sum}}}{\leq}& \tilde{O} \left( \sqrt{ S A H^2 } (K')^{-1/2} + SAH (K')^{-1} \right),
    \end{align*}
    where \cref{lem:sqrt-sum} applies due to assertion (a). For the second term, we first note that $|V_{h+1}^{\pi} (s'; \tilde{l}_k, \hat{p}_{k})| \lesssim H$ since $|\tilde{l}_{k,h}(s,a) | \leq l_h(s,a) + \frac{1}{\sqrt{n^{k-1}_h(s,a) \vee 1}}$ by (b). Hence
    \begin{align*}
        &\sum_{h=1}^H \E\bigg[\sum_{s'} | p_h(s | s_h, a_h) - \hat{p}_{k}(s | s_h, a_h) | V_{h+1}^{\pi} (s'; \tilde{l}_k, \hat{p}_{k}) \mid s_1, p, \pi \bigg]\\
        \overset{(c)}{\lesssim}& H \sum_{h=1}^H \E\bigg[ \frac{1}{\sqrt{n^{k-1}_h(s_h,a_h) \vee 1}} \big(\sum_{s'} \sqrt{p_h(s' | s_h,a_h)}\big) + \frac{S}{n^{k-1}_h(s_h,a_h) \vee 1} \mid s_1, p, \pi \bigg]\\
        \overset{\text{Jensen}}{\leq}& H \sum_{h=1}^H \E\bigg[ \frac{1}{\sqrt{n^{k-1}_h(s_h,a_h) \vee 1}} \sqrt{\mathcal{N}} \sqrt{\sum_{s'} p_h(s' | s_h,a_h)} + \frac{S}{n^{k-1}_h(s_h,a_h) \vee 1} \mid s_1, p, \pi \bigg]\\
        \lesssim& H\sqrt{\mathcal{N}} \cdot \sqrt{SAH^2} (K')^{-1/2} + H\sqrt{\mathcal{N}} \cdot SAH (K')^{-1} + H S \cdot SAH^2 (K')^{-1}\\
        \lesssim& \sqrt{\mathcal{N}SAH^4} (K')^{-1/2} + S^2 A H^3 (K')^{-1},
    \end{align*}
    where the second to last relation holds due to \cref{lem:sqrt-sum} and \cref{lem:sum} (which in turn apply due to assertion (a)).
\end{proof}


\section{Omitted Proofs for \cref{sec:algos}} 


\subsection{Pre-Training Phase} \label{sec:pretrain-omit}

We now establish which constant duration $K'$ of playing $\pibar$ is sufficient to guarantee the desired strict feasibility of \cref{eq:opt-KKT} (with the choice of $K'$ favoring readability over tightness). Note that the conclusion holds conditioned on the success event $G$.

\proppretraining*

\begin{proof}
    Condition on the success event $G$, which by \cref{lem:g-whp} happens with probability $\geq 1- \delta$. Then by construction of $G$, the assumptions (a) and (b) in \cref{lem:onp-error-last} are met for $l=d_i$ for all constraints $i$. By \cref{lem:p-bound} also the assumption (c) in \cref{lem:onp-error-last} is met with $\hat{p}_k := \pt_{k}$ from the $k$-th iteration of \AlgName. 
    Thus, the conclusion from \cref{lem:onp-error-last} holds for $\pi=\pibar$ and $l_h(s,a)= d_{i,h}(s,a)$ (for all $i \in [I]$) and we have for all $i\in [I]$ and $k \geq K'$
    \begin{align*}
        | V^{\pi}( d_i,p) - V^{\pi}( \dt_{i,k}, \pt_{k}) | \leq b (K')^{-1/2} + a (K')^{-1}, 
    \end{align*}
    where
    \begin{align*}
        b = \tilde{O} \left( \sqrt{\mathcal{N}SAH^4} \right), ~~~~~~~~~~~ a = \tilde{O} \left( S^2 A H^3 \right).
    \end{align*}
    Thus, if $K' \geq \max \left \{  \frac{2S^2 A H^3}{(1-\nu)\gamma}, \frac{4\mathcal{N}SAH^4}{(1-\nu)^2\gamma^2} \right \}$, we have 
    \begin{align*}
        | V^{\pibar}( d_i,p) - V^{\pibar}( \dt_{i,k}, \pt_{k}) | \leq \tilde{O}((1-\nu)\gamma).
    \end{align*}
    By further correcting for the missing polylogarithmic terms, this shows that there exists $K' = \tilde{O} \left( \kPrime \right)$ such that for all $i \in [I]$ and $k \geq K'$, we have
    \begin{align*}
        | V^{\pibar}( d_i,p) - V^{\pibar}( \dt_{i,k}, \tilde{p}_{k}) | \leq (1-\nu)\gamma.
   \end{align*}
   At the same time, Assumption \ref{ass:mild-slater} guarantees 
   \begin{align*}
        V^{\pibar}(d_i, p) \leq \alpha_i - \gamma.
   \end{align*}
   Hence
   \begin{align*}
        V^{\pibar}(\dt_{i,k}, \pt_{k}) =& V^{\pibar}(\dt_{i,k}, \pt_{k}) - V^{\pibar}(d_i, p) + V^{\pibar}(d_i, p)\\
        \leq& (1-\nu)\gamma + \alpha_i - \gamma\\
        =& \alpha_i - \nu \gamma,
   \end{align*}
   which proves the claim.
\end{proof}

\begin{remark} \label{rmk:calc-K}
    Notice that we chose $K'$ such that both the terms $b (K')^{-1/2}$ and $a (K')^{-1} $ are less than $1/2$, which is sufficient. For a tighter but less concise bound, we can alternatively solve the quadratic equation
    \begin{align}
        b (K')^{-1/2} + a (K')^{-1} \leq (1-\nu)\gamma
    \end{align}
    for $K'$ to obtain the smallest possible pre-training time, up to polylogarithmic factors. Furthermore, since $a$ and $b$ are only specified in $\tilde{O}$-notation, we only have an asymptotic bound up to polylogarithmic factors on how large to choose $K'$. An exact bound could be established by carrying along all factors dropped in the $\tilde{O}$-notation, but for brevity, we only specify how large to choose $K'$ asymptotically. In practice, one can take the bound above and choose $K'$ slightly larger.
\end{remark}


\subsection{Iteration Complexity of the Inner Method} \label{sec:iter-compl}

Recall \cref{lemma:inner-routine}:

\lemmainnerroutine*

In \cref{cor:set-T}, we prove \cref{lemma:inner-routine} by showing that the method proposed in \cref{sec:inner} achieves the desired iteration complexity for solving \cref{eq:primal-step}. As seen in \cref{sec:fw-inner}, every iteration of the Frank-Wolfe scheme \textsc{InnerOpt-FW} (\cref{algo:step1}) above can be performed efficiently via DP, i.e., in (low-degree) polynomial time in the parameters defining the CMDP. To analyze the number of such Frank-Wolfe iterations needed to reach an $\epsilon_k$-close solution, we recall the following result due to \citet{jaggi2013revisiting}. 

\begin{restatable}[]{theorem}{thmjaggi} \label{thm:jaggi}
    The iterates $(z_k)_{k\geq 0}$ of the Frank-Wolfe algorithm with an exact LMO applied to a convex objective $f \colon \mathcal{D} \to \R$ (with closed, convex and bounded domain $\mathcal{D} \subset \R^d$) satisfy
    \begin{align}
        f(z_k) - f(z^*) \leq \frac{2 C_f}{k+2},
    \end{align}
    where $z^*$ is a minimizer of $f$ over $\mathcal{D}$ and $C_f$ is the curvature constant of $f$.\\
    Moreover, if $\nabla f$ is $L$-Lipschitz continuous w.r.t. an arbitrary norm $\| \cdot \|_{c}$ on $\R^d$, then
    \begin{align}
        C_f \leq \text{diam}_{\| \cdot \|_{c}}(\mathcal{D})^2 L.
    \end{align}
\end{restatable}

In our case, this yields the following bound on the iteration complexity of the inner optimization procedure. 

\begin{corollary} \label{cor:set-T}
    After $T = \frac{2 \eta_k I S^2 A H }{\epsilon_k} $ Frank-Wolfe steps, \textsc{InnerOpt-FW} returns an $\epsilon_k$-close solution to the inner problem (\cref{eq:primal-step}).
\end{corollary}

\begin{proof}
    Consider the problem of minimizing $f$ over $Z$, as defined in \cref{sec:fw-inner}. Let $y, z \in Z$. With respect to $\|\cdot\|_c = \| \cdot \|_{\infty}$ we have, using the expression of the gradients from \cref{eq:nabla},
    \begin{align*}
        &\| \nabla f(y) - \nabla f(z) \|_{\infty} \\
        =&  \| (\nabla f(y))(\cdot,\cdot,1,\cdot) - (\nabla f(z))(\cdot,\cdot,1,\cdot) \|_{\infty}\\
        =& \| \Dt_k^T [\lambda_k + \eta_k ( \Dt_k \sum_{s''} y(s'') - \alpha )]_+ - \Dt_k^T [\lambda_k + \eta_k ( \Dt_k \sum_{s''} z(s'') - \alpha )]_+\|_{\infty} \\
        \leq& \|\Dt_k^T\|_{\infty} \cdot \|[\lambda_k + \eta_k ( \Dt_k \sum_{s''} y(s'') - \alpha )]_+ - [\lambda_k + \eta_k ( \Dt_k \sum_{s''} z(s'') - \alpha )]_+ \|_{\infty}\\
        &\text{and using} [a]_+ - [b]_+ \leq [a-b]_+ \leq |a-b| \\[0.2cm]
        \leq& \|\Dt_k^T\|_1 \cdot \|[\lambda_k + \eta_k ( \Dt_k \sum_{s''} y(s'') - \alpha )] - [\lambda_k + \eta_k ( \Dt_k \sum_{s''} z(s'') - \alpha )] \|_{\infty}\\
        =& \|\Dt_k^T\|_{\infty} \cdot \eta_k \| \Dt_k \sum_{s''} (y(s'')-z(s''))\|_{\infty} \\
        \leq& \|\Dt_k^T\|_{\infty} \cdot \eta_k \| \Dt_k \|_{\infty} \sum_{s''} \|y(s'')-z(s'')\|_{\infty} \\
        \leq& SAH \cdot \eta_k \cdot I \cdot S\|y-z\|_{\infty},
    \end{align*}
    where the final bound holds since the entries of $\Dt_k$ are in $[0,1]$. Thus, $\nabla f (\cdot)$ is $(S^2AH \cdot \eta_k \cdot I)$-Lipschitz continuous and thus $f$ is $(S^2AH \cdot \eta_k \cdot I)$-smooth. Moreover, for all $z\in Z$ we have $0 \leq z_h(s,a,s') \leq 1$ in every component, so for $y,z \in Z$ we have $ \| z - y \|_{\infty} \leq 1$. Thus $\text{diam}_{\|\cdot\|_{\infty}}(Z)^2 \leq 1 $. Plugging both into \cref{thm:jaggi} yields the result. 
\end{proof} 
It may be possible to improve this bound using a norm different from $\| \cdot \|_c = \| \cdot \|_\infty$ for the analysis.


\subsection{Main Result} \label{sec:main-res-omit}

Our main result shows that with probability $1- \delta$, \AlgName~achieves a sublinear regret in the number of episodes $K$ and polylogarithmic in $1/\delta$.

\thmoptaugrlppregret*

\begin{proof}
    According to our regret decomposition (Observation \ref{obs:regret-dec}), we have 
    \begin{align*}
        \Reg (K; c) &\leq K'H {+} \underbrace{\sum_{k=K'+1}^{K} [ V^{\pi_{k}}(c, p) {-} V^{\pi_{k}}( \ct_k, \pt_{k}) ]_{+}}_{\text{Estimation Error}} {+} \underbrace{\sum_{k=K'+1}^{K} [ V^{\pi_{k}}( \ct_k, \pt_{k}) {-} V^{\pi^*}(c, p) ]_{+}}_{\text{Optimization Error}}, \\
        \Reg (K; d) &\leq \underbrace{\max_{i \in [I]} \sum_{k=K'+1}^{K} [ V^{\pi_{k}}( d_i, p) {-} V^{\pi_{k}}(\dt_{i,k}, \pt_{k}) ]_{+}}_{\text{Estimation Error}} {+} \underbrace{\max_{i \in [I]} \sum_{k=K'+1}^{K} [ V^{\pi_{k}}(\dt_{i,k}, \pt_{k}) {-} \alpha_i ]_{+}}_{\text{Optimization Error}}.
    \end{align*}
    Suppose the success event $G$ occurs, which happens with probability at least $1-\delta$ as shown in \cref{lem:g-whp}. By \cref{prop:est-regret}, the estimation errors satisfy
    \begin{align*}
        \sum_{k=K'+1}^{K} [ V^{\pi_{k}}(c, p) - V^{\pi_{k}}( \ct_k, \pt_{k}) ]_{+} \leq& \tilde{O} \left(  \sqrt{\mathcal{N}SAH^4 K} + S^2 A H^3 \right),\\
        \max_{i \in [I]} \sum_{k=K'+1}^{K} \left[ V^{\pi_{k}}( d_i, p) - V^{\pi_{k}}(\dt_{i,k}, \pt_{k}) \right]_{+} \leq&\tilde{O} \left( \sqrt{\mathcal{N}SAH^4 K} + S^2 A H^3 \right).
    \end{align*}
    By \cref{lemma:opt-error-k}, the optimization errors satisfy
    \begin{align*}
        \sum_{k=K'+1}^{K} [ V^{\pi_{k}}( \ct_k, \pt_{k}) - V^{\pi^*}(c, p) ]_{+} \leq& \sum_{k=K'+1}^{K} \bigg( \frac{(O(\sigma k) + \sum_{t=K'+1}^{k} \sqrt{2\eta_t\epsilon_t})^2}{2\eta_k} + \epsilon_k \bigg) \leq O(\sqrt{K}), \\
        \max_{i \in [I]} \sum_{k=K'+1}^{K} [ V^{\pi_{k}}(\dt_{i,k}, \pt_{k}) - \alpha_i ]_{+} \leq& \sum_{k=K'+1}^{K} \frac{O(\sigma k) + \sum_{t=K'+1}^{k}    \sqrt{2\eta_t \epsilon_k}}{\eta_k} \leq O(\sqrt{K}), 
    \end{align*}
    where $\sigma = \frac{H}{\nu\gamma}$ and the term $O(\sigma k)$ can be replaced by $(2+2(k-K'))\sigma$. As shown in \cref{lemma:opt-error-k}, we can choose step sizes $\eta_{K'+k} = \Theta(k^{2.5})$ and $\epsilon_{K'+k} = \Theta(1/\eta_{K'+k})$ such that both of the above bounds are $O(\sqrt{K})$ up to constant factors. Summing up concludes the proof.
\end{proof}

\begin{remark} \label{rmk-main-param}
    Note that while the bound for the estimation errors does not depend on the choices for $\eta_k$ and $\epsilon_k$, the one for the optimization errors does. For any such choice, the general regret bound we obtain reads
    \begin{align}
        \Reg(K; c) &\leq \tilde{O} \left( \sqrt{\mathcal{N}SAH^4 K} + S^2 A H^3  + K' H \right) \nonumber\\
        &+ \sum_{k=K'+1}^{K} \bigg( \frac{(O(\sigma k) + \sum_{t=K'+1}^{k} \sqrt{2\eta_t\epsilon_t})^2}{2\eta_k} + \epsilon_k \bigg), \label{eq:main-param}\\
        \Reg(K;d) &\leq \tilde{O} \left( \sqrt{\mathcal{N}SAH^4 K} + S^2 A H^3 \right) + \sum_{k=K'+1}^{K} \frac{O(\sigma k) + \sum_{t=K'+1}^{k} \sqrt{2\eta_t \epsilon_k}}{\eta_k}, \nonumber
    \end{align}
    where $\sigma = \frac{H}{\nu\gamma}$ and the term $O(\sigma k)$ can be replaced by $ (2+2(k-K'))\sigma$.

    We remark that there are multiple ways to choose $\eta_k$, $\epsilon_k$ in the bound above that yield different regret guarantees and complexities of the inner algorithm \textsc{InnerOpt-FW}. One possible way is to choose a decreasing approximation error $\epsilon_k$ in every episode and $\eta_k$ large enough so that the parameter-dependent terms above are of order $O(\sqrt{K})$. This yields vanishing average regret for both the objective and the constraints. One such choice is\footnote{Note that if $\gamma$ is not known, having an estimate of $\sigma = \frac{H}{\nu \gamma}$ that is larger than the true value is clearly sufficient.}
    \begin{align*}
        \epsilon_{K'+k} &:= \frac{1}{2\eta_{K'+k}} \tab (k \geq 1),\\
        \eta_{K'+k} &:= ((2+3k)\sigma)^{2.5} \tab (k \geq 1),
    \end{align*}
    as we discuss in \cref{lemma:opt-error-k}. 
    
    It is worth noting that when choosing $\eta_k$ very large and $\epsilon_k$ very small, the bound we get on the parameter-dependent terms in \cref{eq:main-param} gets better and even becomes constant when choosing $\eta_k$, $\epsilon_k$ exponential in $k-K'$. However, the inner problem (\cref{eq:primal-step}) we solve in every episode becomes computationally harder. This is due to the changing smoothness of the objective (with respect to the occupancy measure). Indeed, \cref{lemma:inner-routine} quantifies the resulting iteration complexity, formalizing this tradeoff between statistical guarantees and computational efficiency. Thus, in practice, it is advised to choose $\eta_k$ as large ($\epsilon_k$ as small) as needed to ensure sublinear regret but as small as possible to avoid a higher computational cost.
\end{remark}


\section{Omitted Proofs for \cref{sec:regret}} \label{sec:regret-omit}


\subsection{Preliminaries} \label{sec:regret-prelim-omit}


\subsubsection{Preliminary Confidence Bounds} \label{sec:regret-conf-omit} 
Fix an arbitrary number of episodes $K$ and $0 < \delta < 1$, corresponding to a confidence of $1-\delta$. Moreover, let $(\pi_k)_{k=1}^K$ be the policies according to which the optimistic estimates are updated (\cref{sec:algos}).

Define the following confidence sets 
\begin{align*}
    B_{k,h}^{p}(s,a) &:= \left\{ \pt_{h}(\cdot | s,a) \in \Simplex{\St}S \mid \forall s' \in \St \colon |\pt_h(s' | s,a) - \bar{p}_h^{k-1}(s' | s,a) | \leq \beta_{k,h}^{p}(s,a,s') \right\},\\
    B_{k,h}^{c}(s,a) &:= \left[ \bar{c}_h^{k-1}(s,a) - \beta_{k,h}^{c}(s,a), \bar{c}_h^{k-1}(s,a) + \beta_{k,h}^{c}(s,a) \right],\\
    B_{i,k,h}^{d}(s,a) &:= \left[ \bar{d}_{i,h}^{k-1}(s,a) - \beta_{i,k,h}^{d}(s,a), \bar{d}_{i,h}^{k-1}(s,a) + \beta_{i,k,h}^{d}(s,a) \right],
\end{align*}
where (derived using Hoeffding for the costs and empirical Bernstein for the transitions)
\begin{align*}
    \beta_{k,h}^{p}(s,a,s') :=& 2 \sqrt{\frac{\bar{p}_h^{k-1}(s' | s,a)(1-\bar{p}_h^{k-1}(s' | s,a))L_{\delta}^p}{n_{h}^{k-1}(s,a) \vee 1}} + \frac{\frac{14}{3} L_{\delta}^p }{n_{h}^{k-1}(s,a) \vee 1},\\
    \beta_{k,h}^{c}(s,a) := \beta_{i,k,h}^{d}(s,a) :=& \sqrt{\frac{L_{\delta}}{n_{h}^{k-1}(s,a) \vee 1}},
\end{align*}
with 
\begin{align*}
    L_{\delta}^{p} := \log\left( \frac{6SAHK}{\delta} \right) &\tab  L_{\delta} := \log\left( \frac{6SAH(I+1)K}{\delta} \right).
\end{align*}
We want that, with probability at least $1-\delta$, the true MDP is contained within these bounds across all $k \in [K]$. This will be needed for bounding the optimization error.

With the same thresholds $\beta_{k,h}^{p}$, $\beta_{k,h}^{c}$, $\beta_{i,k,h}^{d}$ as above we define the following \textit{failure events}
\begin{align*}
    F_k^p &:= \left\{ \exists s,a,s',h \colon |p_{h}(s' | s,a) - \bar{p}_h^{k-1}(\cdot | s,a) | > \beta_{k,h}^{p}(s,a,s') \right\}, \\
    F_k^c &:= \left\{ \exists s,a,h \colon |\bar{c}_h^{k-1}(s,a) - c_h(s,a)| > \beta_{k,h}^{c}(s,a) \right\}, \\
    F_k^d &:= \left\{ \exists s,a,i,h \colon |\bar{d}_{i,h}^{k-1}(s,a) - d_{i,h}(s,a)| > \beta_{i,k,h}^{d}(s,a) \right\}, \\
    F_k^N &:= \left\{ \exists s,a,h \colon n_h^{k-1}(s,a) \leq \sum_{j<k} q_{h}^{\pi_j}(s,a;p) - H \log\left( \frac{SAH}{\delta'} \right) \right\},
\end{align*}
where we set $\delta' := \delta / 3$. With this notation, set $F^p := \cup_{k=1}^K F^p_k$, $F^c := \cup_{k=1}^K F^c_k$, $F^d := \cup_{k=1}^K F^d_k$, $F^N := \cup_{k=1}^K F^N_k$ and finally denote the event that none of the failure events ever occurs by
\begin{align*}
    G := \overline{ \left( F^p \bigcup F^c \bigcup F^d \bigcup F^N \right) },
\end{align*}
which we will refer to as the \textit{success event}.

We immediately have the following by construction $G$ and the optimistic estimates.
\begin{restatable}[]{observation}{obsincluded} \label{obs:included}
    Conditioned on the success event $G$, for every $k \in [K]$ the CMDP with cost $\ct_k$ and constraint costs $(\dt_{i,k})_{i\in[I]}$ is optimistic and all $B_k^p(s,a)$ contain the true respective transition probability distribution, i.e.,
    \begin{align*}
        \ct_k \leq c, \tab \dt_{i,k} \leq d_{i}, \tab \text{and} \tab p \in B_k^p,
    \end{align*}
    where $B_k^p := \{ \pt \mid \forall s,a,h \colon \pt_h(\cdot | s,a) \in B_{k,h}^p(s,a) \}$ is the set of plausible transition probabilities.
\end{restatable}

Concentration bounds on the individual random variables and a union bound over all indices now allow us to establish the following result. In \cref{sec:main-res-omit}, we show that conditioned on $G$, the regrets are sublinear in $K$ and polylogarithmic in $1/\delta$.

\lemgwhp*

\begin{proof}
    \citet[Appendix A.1]{efroni2020exploration} give a proof of this. Note that the proof does not require any specific properties of the used policy iterates $(\pi_k)_{k\in [K]}$ but only uses that the collected costs and transitions are i.i.d. across episodes.
\end{proof}

Moreover, we have the following result, allowing us to bound the estimation error. The absolute constants could be specified via a short calculation, but we omit this for brevity.

\begin{restatable}[Lemma 8, \citet{jin2019learning}]{lemma}{lempbound} \label{lem:p-bound}
    In the setup above, conditioned on $G$, there exist absolute constants $C_1, C_2 > 0$ such that for all indices $k, h, s, a, s'$ we have 
    \begin{align*}
        |p_{h}(s' | s,a) - \bar{p}_h^{k-1}(s' | s,a) | \leq C_1 \sqrt{\frac{p_h(s'| s,a)L_{\delta}^p}{n_{h}^{k-1}(s,a) \vee 1}} + C_2 \frac{L_{\delta}^p}{n_{h}^{k-1}(s,a) \vee 1},
    \end{align*}
    with $L_{\delta}^{p} = \log\left( \frac{6SAHK}{\delta} \right)$ as before. 
\end{restatable}

\begin{proof}
    We refer to \citet[Lemma 8]{jin2019learning} for a proof.
\end{proof}


\subsubsection{Regret Decomposition}
The simple observation that we can decompose the regrets can be seen as follows.

\begin{restatable}[Regret decomposition]{observation}{obsregretdecomp} \label{obs:regret-dec}
    \begin{align*}
        \Reg (K; c) &\leq K'H {+} \underbrace{\sum_{k=K'+1}^{K} [ V^{\pi_{k}}(c, p) {-} V^{\pi_{k}}( \ct_k, \pt_{k}) ]_{+}}_{\text{Estimation Error}} {+} \underbrace{\sum_{k=K'+1}^{K} [ V^{\pi_{k}}( \ct_k, \pt_{k}) {-} V^{\pi^*}(c, p) ]_{+}}_{\text{Optimization Error}} \\
        \Reg (K; d) &\leq \underbrace{\max_{i \in [I]} \sum_{k=K'+1}^{K} [ V^{\pi_{k}}( d_i, p) {-} V^{\pi_{k}}(\dt_{i,k}, \pt_{k}) ]_{+}}_{\text{Estimation Error}} {+} \underbrace{\max_{i \in [I]} \sum_{k=K'+1}^{K} [ V^{\pi_{k}}(\dt_{i,k}, \pt_{k}) {-} \alpha_i ]_{+}}_{\text{Optimization Error}}
    \end{align*}
\end{restatable}

\begin{proof}
    We first split the regrets between the two phases of the algorithm.
    \begin{align*}
        \Reg (K; c) &= \underbrace{\sum_{k=1}^{K'} [ V^{\pi_{k}}( c, p) - V^{\pi^*}(c,p) ]_+ }_{\text{Pre-Training, } \leq K'H} + \underbrace{\sum_{k=K'+1}^{K} [ V^{\pi_{k}}( c, p) - V^{\pi^*}(c,p) ]_+ }_{\text{Optimistic Exploration}} \\
        \Reg (K; d) &\leq  \underbrace{\max_{i \in [I]} \sum_{k=1}^{K'} [V^{\pi_{k}}( d_i, p) - \alpha_i]_+}_{\text{Pre-Training, } = 0} + \underbrace{\max_{i \in [I]} \sum_{k=K'+1}^{K} [V^{\pi_{k}}( d_i, p) - \alpha_i]_+}_{\text{Optimistic Exploration}}
    \end{align*}
     We can trivially bound the objective regret during the pre-training phase by $K' H$ (since the expected costs are in $[0,1]$ and the time horizon is $H$, so the value functions are in $[0,H]$). Since $\pibar$ is strictly feasible, there is no constraint regret during pre-training. 
     
     We now focus on the regrets incurred in the optimistic exploration phase. We split the sum and use that $[a+b]_+ \leq [a]_++[b]_+$. For the objective, regret we have 
    \begin{align*}
        \Reg (K; c) =& K' H + \sum_{k=K'+1}^{K} [ V^{\pi_{k}}(c, p) - V^{\pi^*}(c, p) ]_{+}   \\
        =& K' H + \sum_{k=K'+1}^{K} [ V^{\pi_{k}}(c, p) - V^{\pi_{k}}( \ct_k, \pt_{k}) + V^{\pi_{k}}( \ct_k, \pt_{k}) - V^{\pi^*}(c, p) ]_{+}  \\
        \leq& K' H + \sum_{k=K'+1}^{K} [ V^{\pi_{k}}(c, p) - V^{\pi_{k}}( \ct_k, \pt_{k}) ]_{+} + \sum_{k=K'+1}^{K} [ V^{\pi_{k}}( \ct_k, \pt_{k}) - V^{\pi^*}(c, p) ]_{+}  
    \end{align*}
    and similarly, for the constraint regret, we get 
    \begin{align*}
        \Reg (K; d) =& 0 + \max_{i \in [I]} \left(\sum_{k=K'+1}^{K}\left[ V^{\pi_{k}}( d_i, p) - \alpha_i \right]_{+} \right)\\
        =& \max_{i \in [I]} \sum_{k=K'+1}^{K} \left[ V^{\pi_{k}}( d_i, p) - \alpha_i - (V^{\pi_{k}}(\dt_{i,k}, \pt_{k}) - \alpha_i) + (V^{\pi_{k}}(\dt_{i,k}, \pt_{k}) - \alpha_i) \right]_{+} \\
        \leq& \max_{i \in [I]} \left(\sum_{k=K'+1}^{K} \left[ V^{\pi_{k}}( d_i, p)  - V^{\pi_{k}}(\dt_{i,k}, \pt_{k}) \right]_{+} + \sum_{k=K'+1}^{K} \left[ V^{\pi_{k}}(\dt_{i,k}, \pt_{k}) - \alpha_i \right]_{+}  \right)\\
        \leq& \max_{i \in [I]} \sum_{k=K'+1}^{K} \left[ V^{\pi_{k}}( d_i, p)  - V^{\pi_{k}}(\dt_{i,k}, \pt_{k}) \right]_{+} + \max_{i \in [I]} \sum_{k=K'+1}^{K} \left[ V^{\pi_{k}}(\dt_{i,k}, \pt_{k}) - \alpha_i \right]_{+}.  
    \end{align*}
\end{proof}


\subsection{Omitted Proofs for \cref{sec:term1}} \label{sec:term1-omit}

We leverage the on-policy error bound from \cref{lem:onp-error} and the preliminaries from \cref{sec:regret-conf-omit} to establish the desired bound on the estimation errors.

\propestregret*

\begin{proof}
    Condition on the success event $G$. Then by construction of $G$, the assumptions (a) and (b) in \cref{lem:onp-error} are met for $l=c$ and $l=d_i$ for all constraints $i\in[I]$. By \cref{lem:p-bound}, also assumption (c) in \cref{lem:onp-error} is met with $\hat{p}_k := \pt_{k}$ from the $k$-th iteration of \AlgName. 

    First, consider the terms $[ V^{\pi_{k}}(c, p) - V^{\pi_{k}}( \ct_k, \pt_{k}) ]_{+}$. For the cost $l=c$, the values of the true and the estimated MDPs are $V^{\pi_{k}}(c,p)$ and $V^{\pi_{k}}(\ct_k,\pt_{k}) $, respectively. Thus, invoking \cref{lem:onp-error} we find
    \begin{align*}
        \sum_{k=K'+1}^{K} [ V^{\pi_{k}}(c, p) - V^{\pi_{k}}( \ct_k, \pt_{k}) ]_{+} \leq& \sum_{k=K'+1}^{K} | V^{\pi_{k}}(c, p) - V^{\pi_{k}}( \ct_k, \pt_{k}) |\\
        \leq& \sum_{k=1}^{K} | V^{\pi_{k}}(c,p) - V^{\pi_{k}}(\ct_k,\pt_{k}) |\\ 
        \leq& \tilde{O} \left( \sqrt{\mathcal{N}SAH^4 K} + S^2 A H^3 \right).
    \end{align*}
    Similarly, let $i\in [I]$ and consider the terms $[ V^{\pi_{k}}( d_i, p) - V^{\pi_{k}}(\dt_{i,k}, \pt_{k}) ]_{+}$. For the cost $l=d_i$, the values of the true and the estimated MDPs are $V^{\pi_{k}}(d_i,p) $ and $V^{\pi_{k}}(\dt_{i,k},\pt_{k})$, respectively. Thus, invoking \cref{lem:onp-error} we find
    \begin{align*}
        \sum_{k=K'+1}^{K} [ V^{\pi_{k}}( d_i, p) - V^{\pi_{k}}(\dt_{i,k}, \pt_{k}) ]_{+} \leq& \sum_{k=1}^{K} | V^{\pi_{k}}( d_i, p) - V^{\pi_{k}}(\dt_{i,k}, \pt_{k}) |\\
        \leq& \tilde{O} \left( \sqrt{\mathcal{N}SAH^4 K} + S^2 A H^3 \right).
    \end{align*}
\end{proof}


\subsection{Omitted Proofs for \cref{sec:term2}} \label{sec:term2-omit}

For notational convenience, we write $D = (d_i)_{i\in[I]}$ and $\Dt_k = (\dt_{i,k})_{i\in[I]}$ throughout this section.


\subsubsection{Preliminary Bounds}

Recall that in \AlgName, we find an $\epsilon_k$-close solution to the inner problem \cref{eq:primal-step} in every episode $k\in \curly{K'+1, \dots, K}$.

\lemmafdiff*

\begin{proof}
    We first bound the error for the objective cost. Let $k \in \curly{K'+1, \dots, K}$. Conditioned on the success event, we have (componentwise) $c \geq \ct_{k}$, $D \geq \Dt_{k}$ and $p \in B^p_k$ (Observation \ref{obs:included}), and thus 
    \begin{align*}
        &V^{\pi^*}(c, p) \\
        =& \left( V^{\pi^*}(c,p) + \frac{1}{2\eta_k} \| [ \lambda_k + \eta_k (V^{\pi^*}(D,p) - \alpha) ]_+ \|^2 \right)
        -\frac{1}{2\eta_k} \| [ \lambda_k + \eta_k (V^{\pi^*}(D,p) - \alpha) ]_+ \|^2 \\
        \geq& \left( V^{\pi^*}(\ct_k,p) + \frac{1}{2\eta_k} \| [ \lambda_k + \eta_k (V^{\pi^*}(\Dt_k,p) - \alpha) ]_+ \|^2 \right)
        -\frac{1}{2\eta_k} \| [ \lambda_k + \eta_k (V^{\pi^*}(D,p) - \alpha) ]_+ \|^2 \\ 
        \geq& \min_{\substack{\pi \in \Pi \\ p' \in B_k^p}} \left( V^{\pi}(\ct_k, p') + \frac{1}{2\eta_k} \| [ \lambda_k + \eta_k ( V^{\pi}(\Dt_k,p') - \alpha ) ]_+ \|^2 \right) \\
        &-\frac{1}{2\eta_k} \| [ \lambda_k + \eta_k (V^{\pi^*}(D,p) - \alpha) ]_+ \|^2 \\ 
        \geq& \left( V^{\pi_{k}}(\ct_k, \pt_{k}) + \frac{1}{2\eta_k} \| [ \lambda_k + \eta_k (  V^{\pi_{k}}(\Dt_k, \pt_{k}) - \alpha ) ]_+ \|^2 - \epsilon_k \right)\\
        &-\frac{1}{2\eta_k} \| [ \lambda_k + \eta_k (V^{\pi^*}(D,p) - \alpha) ]_+ \|^2 \\ 
        =& V^{\pi_{k}}( \ct_k, \pt_{k}) - \epsilon_k + \frac{1}{2\eta_k} \| [ \lambda_k + \eta_k ( V^{\pi_{k}}(\Dt_k, \pt_{k}) - \alpha ) ]_+ \|^2\\
        &-\frac{1}{2\eta_k} \| [ \lambda_k + \eta_k (V^{\pi^*}(D, p) - \alpha) ]_+ \|^2,
    \end{align*}
    where the first inequality is due to optimism and monotonicity of $\|[\cdot]_+\|^2$, the second due to $p \in B^p_k$, and the third due to the $\epsilon_k$-closeness of $\pi_{k}$, $\pt_{k}$. Now since $\pi^*$ is primal-feasible we have $V^{\pi^*}(D, p) - \alpha \leq 0$ and thus $\| [ \lambda_k + \eta_k (V^{\pi^*}(D, p) - \alpha) ]_+ \|^2 \leq \|\lambda_k \|^2 $ by monotonicity of $\|[\cdot]_+\|^2$. Plugging in the update for $\lambda_{k+1}$ thus shows
    \begin{align*}
        &\frac{1}{2\eta_k} \| [ \lambda_k + \eta_k ( V^{\pi_{k}}(\Dt_k, \pt_{k}) - \alpha ) ]_+ \|^2 -\frac{1}{2\eta_k} \| [ \lambda_k + \eta_k ( V^{\pi^*}(D, p) - \alpha) ]_+ \|^2\\
        \geq&\frac{1}{2\eta_k} \| [ \lambda_k + \eta_k (  V^{\pi_{k}}(\Dt_k, \pt_{k}) - \alpha ) ]_+ \|^2 - \frac{1}{2\eta_k} \| \lambda_k \|^2 \\
        =& \frac{1}{2\eta_k} \| \lambda_{k+1}\|^2 - \frac{1}{2\eta_k} \| \lambda_k \|^2 ,
    \end{align*}
    concluding the proof after plugging this into the previous inequality and rearranging. 

    We now proceed by proving the bound on the constraint cost.
    Let $k \in \curly{K'+1, \dots, K}$. Let $i \in [I]$ and recall that $\lambda_{k+1}=[ \lambda_k + \eta_k ( V^{\pi_{k}}(\Dt_k, \pt_{k}) - \alpha ) ]_+ $. If $\lambda_{k+1}(i) > 0$, then we have $\lambda_{k+1}(i) - \lambda_k(i) = \eta_k (V^{\pi_{k}}(\dt_{i,k}, \pt_{k}) - \alpha_i) $ with equality. On the other hand, if $\lambda_{k+1}(i) = 0$, then $\lambda_k(i) + \eta_k (V^{\pi_{k}}(\dt_{i,k}, \pt_{k}) - \alpha_i) \leq 0$ and thus $\lambda_{k+1}(i) - \lambda_k(i) = - \lambda_k(i) \geq (\lambda_k(i) + \eta_k (V^{\pi_{k}}(\dt_{i,k}, \pt_{k}) - \alpha_i)) - \lambda_k(i) = \eta_k (V^{\pi_{k}}(\dt_{i,k}, \pt_{k}) - \alpha_i)$. Thus in both cases $\lambda_{k+1}(i) - \lambda_k(i) \geq \eta_k (V^{\pi_{k}}(\dt_{i,k}, \pt_{k}) - \alpha_i) $, which proves the claim.
\end{proof}


\subsubsection{Bounding the Dual Iterates} \label{sec:dual-iter-bound}

In order to prove \cref{prop:dual-bound}, which implies \cref{cor:dual-bound}, we first establish (\cref{lemma:1,lemma:2,lemma:ialm-step}), following the analysis of \citet{xu2021iteration}.

\begin{lemma} \label{lemma:1}
    Let $k\in\curly{K'+1, \dots, K}$. In the $k$-th iteration of \AlgName~we have, for all $\lambda \in \R_{\geq 0}^I$,
    \begin{align*}
        &\frac{1}{2\eta_k} \left( \| \lambda_{k+1} - \lambda\|^2 - \| \lambda_{k} - \lambda \|^2 + \| \lambda_{k+1} - \lambda_{k}\|^2 \right) \\
        &= \sum_{i=1}^{I} (\lambda_{k+1}(i) - \lambda(i)) \max\{-\frac{\lambda_k(i)}{\eta_k}, (V^{\pi_{k}}(\Dt_k, \pt_{k}) - \alpha)_i\}.
    \end{align*}
\end{lemma}

\begin{proof}
    Using $2u^Tv = \|u\|^2 + \|v\|^2 - \|u-v\|^2$ with $u = \lambda_{k+1}-\lambda_k$, $v=\lambda_{k+1}-\lambda$ we find that the LHS in the Lemma reads
    \begin{align*}
        \frac{1}{2\eta_k} \left( \| \lambda_{k+1} - \lambda\|^2 - \| \lambda_{k} - \lambda \|^2 + \| \lambda_{k+1} - \lambda_{k}\|^2 \right) = \frac{1}{2\eta_k} 2(\lambda_{k+1}-\lambda)^T(\lambda_{k+1}-\lambda_k).
    \end{align*}
    The update rule of for $\lambda_{k+1}$ can equivalently be written as
    \begin{align*}
        \lambda_{k+1}(i) = \lambda_k(i) + \eta_k \max\{ -\frac{\lambda_k(i)}{\eta_k}, (V^{\pi_{k}}(\Dt_{k}, \pt_{k}) - \alpha)_i \},    
    \end{align*}
    for all $i \in [I]$, from which we obtain $\lambda_{k+1}(i) - \lambda_k(i) =  \eta_k \max\{ -\frac{\lambda_k(i)}{\eta_k}, (V^{\pi_{k}}(\Dt_k, \pt_{k}) - \alpha)_i \}$. Plugging this into the equation above proves the Lemma.
\end{proof}

\begin{lemma} \label{lemma:2}
    Let $k\in\curly{K'+1, \dots, K}$. In the $k$-th iteration of \AlgName~we have for, any $\lambda \in \R_{\geq 0}^I$,
    \begin{align*}
        &\sum_{i=1}^{I} (\lambda_{k+1}(i) - \lambda(i)) \max\{-\frac{\lambda_k(i)}{\eta_k}, (V^{\pi_{k}}(\Dt_k, \pt_{k}) - \alpha)_i\}\\
        &\leq \sum_{i=1}^{I} (\lambda_{k+1}(i) - \lambda(i)) (V^{\pi_{k}}(\Dt_k, \pt_{k}) - \alpha)_i.
    \end{align*}
\end{lemma}

\begin{proof}
    Fix $k$ and set $I_+ := \{ i \in [I] \mid \lambda_k(i) + \eta_k (V^{\pi_{k}}(\Dt_k, \pt_{k}) - \alpha)_i > 0 \} = \{ i \in [I] \mid -\frac{\lambda_k(i)}{\eta_k} < (V^{\pi_{k}}(\Dt_k, \pt_{k}) - \alpha)_i \}$ and $I_- := [I] \setminus I_+$. Then subtracting the RHS from the LHS in the Lemma we get
    \begin{align*}
        & \sum_{i=1}^{I} (\lambda_{k+1}(i) - \lambda(i)) \max\{-\frac{\lambda_k(i)}{\eta_k}, (V^{\pi_{k}}(\Dt_k, \pt_{k}) - \alpha)_i\} \\
        &- \sum_{i=1}^{I} (\lambda_{k+1}(i) - \lambda(i)) (V^{\pi_{k}}(\Dt_k, \pt_{k}) - \alpha)_i\\
        =& \sum_{i \in I_+} (\lambda_{k+1}(i) - \lambda(i))((V^{\pi_{k}}(\Dt_k, \pt_{k}) - \alpha)_i - (V^{\pi_{k}}(\Dt_k, \pt_{k}) - \alpha)_i) \\
        &+ \sum_{i \in I_-} \lambda(i) ((V^{\pi_{k}}(\Dt_k, \pt_{k}) - \alpha)_i + \frac{\lambda_k(i)}{\eta_k})\\
        =& \sum_{i \in I_-} \lambda(i) ((V^{\pi_{k}}(\Dt_k, \pt_{k}) - \alpha)_i + \frac{\lambda_k(i)}{\eta_k}),
    \end{align*}
    and using that $\lambda \geq 0$ as well as $(V^{\pi_{k}}(\Dt_k, \pt_{k}) - \alpha)_i + \frac{\lambda_k(i)}{\eta_k} \leq 0$ for $i \in I_-$ shows that this is $\leq 0$, which proves the claim.
\end{proof}

Combining the previous two lemmas, we find the analogous result of the one-step progress inequality of the standard inexact augmented Lagrangian method analysis:
\begin{lemma}[One-step progress of iALM] \label{lemma:ialm-step}
    Let $k\in\curly{K'+1, \dots, K}$. In step $k$ of \AlgName, for all $\pi \in \Pi$ and all $\lambda \geq 0$ we have 
    \begin{align*}
        V^{\pi_{k}}( \ct_k, \pt_{k}) + \lambda^T (V^{\pi_{k}}(\Dt_k, \pt_{k}) - \alpha) + \frac{1}{2\eta_k} \| \lambda_{k+1} - \lambda \|^2 \\
        \leq \Ltk (\pi, \lambda_k) + \frac{1}{2\eta_k} \| \lambda_k - \lambda \|^2 + \epsilon_k,
    \end{align*}
    where for notational convenience
    \begin{align*}
        \Ltk(\pi,\lambda) := V^{\pi}(\ct_k, \pt_{k}) + \frac{1}{2\eta_k} \| [\lambda + \eta_k (V^{\pi}(\Dt_k, \pt_{k}) - \alpha)]_+ \|^2 - \frac{1}{2\eta_k} \|\lambda\|^2.
    \end{align*}
\end{lemma}

\begin{proof}
    By $\epsilon_k$-closeness, for all $\pi \in \Pi$ we have
    \begin{align}
        V^{\pi_{k}}( \ct_k, \pt_{k}) + \frac{1}{2\eta_k} \|\lambda_{k+1}\| \leq& V^{\pi}(\ct_k, \pt_{k}) + \frac{1}{2\eta_k}  \| \left[ \lambda_k + \eta_k ( V^{\pi}(\Dt_k, \pt_{k}) - \alpha ) ]_+ \right \|^2  + \epsilon_k \nonumber\\
        =& \Ltk (\pi, \lambda_k) + \frac{1}{2\eta_k} \| \lambda_k \|^2 + \epsilon_k. \label{eq:lagr1}
    \end{align}
    In addition, again by distinguishing between indices in $I_+$ and $I_-$ (see proof of \cref{lemma:2}), we see
    \begin{align*}
        & \frac{1}{2\eta_k} \| [\lambda_k + \eta_k (V^{\pi_{k}}(\Dt_k, \pt_{k}) - \alpha)]_+ \|^2 - \frac{1}{2\eta_k}\|\lambda_{k}\|^2 \\
        &- \sum_{i\in [I]} [\lambda_k(i) + \eta_k (V^{\pi_{k}}(\Dt_k, \pt_{k}) - \alpha)_i]_+ (V^{\pi_{k}}(\Dt_k, \pt_{k}) - \alpha)_i \\
        =& \frac{1}{2\eta_k} \sum_{i \in I_+} \bigg( (\lambda_k(i) + \eta_k(V^{\pi_{k}}(\Dt_k, \pt_{k}) - \alpha)_i)^2 - \lambda_k(i)^2 \\
        &- 2(\lambda_k(i) + \eta_k(V^{\pi_{k}}(\Dt_k, \pt_{k}) - \alpha)_i) \eta_k (V^{\pi_{k}}(\Dt_k, \pt_{k}) - \alpha)_i \bigg) - \frac{1}{2\eta_k} \sum_{i\in I_{-}} \lambda_k(i)^2\\
        =& \frac{1}{2\eta_k} \sum_{i \in I_+}  -(\eta_k(V^{\pi_{k}}(\Dt_k, \pt_{k}) - \alpha)_i)^2 - \frac{1}{2\eta_k} \sum_{i\in I_{-}} \lambda_k(i)^2 \\
        =& \frac{1}{2\eta_k} \sum_{i \in I_+}  -(\lambda_{k+1}(i) - \lambda_{k}(i))^2 - \frac{1}{2\eta_k} \sum_{i\in I_{-}} (\lambda_{k+1}(i) - \lambda_{k}(i))^2 \\
        =& - \frac{1}{2\eta_k} \| \lambda_{k+1} - \lambda_k \|^2,
    \end{align*}
    and with \cref{eq:lagr1} this shows
    \begin{align}
        &V^{\pi_{k}}( \ct_k, \pt_{k}) + \sum_{i\in [I]} [\lambda_k(i) + \eta_k (V^{\pi_{k}}(\Dt_k, \pt_{k}) - \alpha)_i]_+ (V^{\pi_{k}}(\Dt_k, \pt_{k}) - \alpha)_i \nonumber\\
        &- \frac{1}{2\eta_k} \| \lambda_{k+1} - \lambda_k \|^2 \nonumber\\
        =& V^{\pi_{k}}( \ct_k, \pt_{k}) + \frac{1}{2\eta_k} \| \lambda_{k+1} \|^2 - \frac{1}{2\eta_k}\|\lambda_{k}\|^2 \nonumber \\
        \overset{\cref{eq:lagr1}}{\leq}& \Ltk (\pi, \lambda_k) + \epsilon_k. \label{eq:lagr2}
    \end{align}
    Now combining \cref{lemma:1} and \cref{lemma:2}, we also have
    \begin{align*}
        &\frac{\| \lambda_{k+1} - \lambda\|^2 - \| \lambda_{k} - \lambda \|^2 + \| \lambda_{k+1} - \lambda_{k}\|^2}{2\eta_k} \\
        \leq &\sum_{i\in [I]} ( [\lambda_k(i) + \eta_k (V^{\pi_{k}}(\Dt_k, \pt_{k}) - \alpha)_i]_+ - \lambda(i) )(V^{\pi_{k}}(\Dt_k, \pt_{k}) - \alpha)_i.
    \end{align*}
    Hence 
    \begin{align*}
        &V^{\pi_{k}}( \ct_k, \pt_{k}) + \frac{1}{2\eta_k}\|\lambda_{k+1}-\lambda\|^2 - \frac{1}{2\eta_k}\|\lambda_k-\lambda\|^2\\
        \leq& V^{\pi_{k}}( \ct_k, \pt_{k}) \\
        &+ \sum_{i\in [I]} ( [\lambda_k(i) + \eta_k (V^{\pi_{k}}(\Dt_k, \pt_{k}) - \alpha)_i]_+ - \lambda(i) )(V^{\pi_{k}}(\Dt_k, \pt_{k}) - \alpha)_i \\
        &- \frac{1}{2\eta_k} \| \lambda_{k+1} - \lambda_{k}\|^2\\
        \overset{\cref{eq:lagr2}}{\leq}& -\sum_{i\in [I]} \lambda(i) (V^{\pi_{k}}(\Dt_k, \pt_{k}) - \alpha)_i + \Ltk (\pi, \lambda_k) + \epsilon_k,
    \end{align*}
    and rearranging this yields the desired inequality.
\end{proof}

Before proving \cref{cor:dual-bound}, we must finally establish \cref{lemma:aux-ineq}. This result is based on standard properties from constrained convex optimization (\cref{sec:conv-prelim}) and the LP formulation of CMDPs (\cref{sec:occ-notation}). Recall that we have $V^{\pi}(\ct_k, \pt_{k}) = \ct_k^T q^{\pi}(\pt_{k})$ and $V^{\pi}(\Dt_k, \pt_{k}) = \Dt_k q^{\pi}(\pt_{k})$, where $q^{\pi}(\pt_{k}) \in \R^{SAH}$, $\ct_k \in \R^{SAH}$, $\dt_{i,k} \in \R^{SAH}$, and $\Dt_{k} \in \R^{I \times SAH}$ are defined as in \cref{sec:occ-notation}. Thus, when switching to occupancy measures, the optimistic problem (\cref{eq:opt-KKT}) equivalently reads
\begin{align}
     \min_{\substack{q^\pi \in Q}(\pt_{k})} ~~~~ \ct_k^T q^{\pi} ~~~~ \text{s.t.} ~~~~ \Dt_k q^{\pi} - \alpha \leq 0, \label{eq:opt-KKT-occ}
\end{align}
where $Q(\pt_{k}) \subset \R^{SAH}$ is defined as in \cref{sec:occ-notation}. This is a convex optimization problem over the set of occupancy measures $Q(\pt_{k})$. 

Formally, we make the following assumption. By \cref{prop:pretraining}, the assumption will be guaranteed to hold with (with $\pi^0 = \pibar$ and $\sigma = \frac{H}{\nu \gamma}$), conditioned on $G$.
\begin{restatable}[Slater points]{assumption}{assslater} \label{ass:slater}
    There exists $\pi^{0} \in \Pi$ such that for all $k \in \curly{K', \dots, K}$ and all $i\in [I]$ we have $V^{\pi^{0}}( \dt_{i,k}, \pt_{k}) < \alpha_i$. In particular, \cref{eq:opt-KKT} is feasible. Let $\pi^*_k$ be an optimal solution for \cref{eq:opt-KKT} and suppose there is a fixed constant $\sigma > 0$ such that, for all $k \in \curly{K', \dots, K}$,
     \begin{align*}
        \frac{V^{\pi^0}(\ct_k, \pt_{k}) - V^{\pi^*_k}(\ct_k, \pt_{k})}{\min_{i \in [I]} (\alpha_i - V^{\pi^0}( \dt_{i,k}, \pt_{k}))} \leq \sigma.
    \end{align*}
\end{restatable}

\begin{remark} \label{rmk:match-convex}
    Note that, under Assumption \ref{ass:slater}, we can view \cref{eq:opt-KKT} as the convex optimization problem in \cref{eq:opt-KKT-occ} over $Q(\pt_{k})$ that satisfies all parts of Assumption \ref{ass:841} from \cref{sec:conv-prelim}. Indeed, 
    \begin{itemize}
        \item[(a)] $X:=Q(\pt_{k})$ is a polytope and thus convex
        \item[(b)] the objective $f(\cdot) := \ct_{k}^T (\cdot)$ is affine and thus convex
        \item[(c)] the constraints $g_i(\cdot) := \dt_{i,k}^T (\cdot) - \alpha_i$ are affine and thus convex
        \item[(d)] by Assumption $\ref{ass:slater}$, \cref{eq:opt-KKT-occ} is feasible, and thus its minimum is attained (since the domain is compact and the objective continuous)
        \item[(e)] a Slater point exists by Assumption \ref{ass:slater}, namely $q^{\pi^0}(\pt_{k})$
        \item[(f)] all dual problems have an optimal solution since the domain $X$ is compact and the objective $f(\cdot) + \lambda^Tg(\cdot)$ is continuous,
    \end{itemize} 
    where $Q(\pt_{k}) \subset \R^{SAH}$, $\ct_k \in \R^{SAH}$ and $\dt_{i,k} \in \R^{SAH}$ are defined as in \cref{sec:occ-notation}, and the setup is in line with the general convex optimization setup described in \cref{sec:conv-prelim}.
\end{remark}

\begin{lemma} \label{lemma:aux-ineq}
     If \cref{eq:opt-KKT-occ} is stricly feasible, then there exists a point $(\pi_k^*, \lambda_k^*)$ with $q^{\pi_k^*} = q^{\pi_k^*}(\pt_{k})$ optimal for \cref{eq:opt-KKT-occ} and $\lambda^*_k$ optimal for its dual problem. For any such pair and any $\pi\in \Pi$ we have
     \begin{align*}
        V^{\pi}(\ct_k, \pt_{k}) - V^{\pi_k^*}(\ct_k, \pt_{k}) + (\lambda_k^*)^T (V^{\pi}(\Dt_k, \pt_{k}) - \alpha) \geq 0.
     \end{align*}
\end{lemma}

\begin{proof}
    By \cref{rmk:match-convex}, we are in the setup of \cref{eq:opt-P} under Assumption \ref{ass:841}. Thus, by applying \cref{thm:conv-aux-ineq}, in the notation of \cref{sec:occ-notation} we immediately get
    \begin{align*}
        \ct_k^T q^{\pi}(\pt_{k}) - \ct_k^T q^{\pi_k^*}(\pt_{k}) + (\lambda_k^*)^T (\Dt_k q^{\pi}(\pt_{k}) - \alpha) \geq 0, 
    \end{align*}
    which proves the claim by plugging in the value functions. 
\end{proof}

The following lemma allows us to deduce \cref{cor:dual-bound}.

\begin{restatable}[]{lemma}{propdualbound}\label{prop:dual-bound}
    Let $k \in \curly{K'+1, \dots, K}$ and suppose \cref{eq:opt-KKT} is strictly feasible. Let $(\pi_k^*, \lambda_k^*)$ be a pair of primal-optimal and dual-optimal solutions for \cref{eq:opt-KKT} (see \cref{lemma:aux-ineq}). Then the iterates of \AlgName~satisfy
    \begin{align*}
        \| \lambda_{k+1} - \lambda_k^* \|^2 \leq& \| \lambda_k - \lambda_k^* \|^2 + 2\eta_k\epsilon_k.
    \end{align*}
\end{restatable}

\begin{proof}
    Fix an arbitrary point $(\pi_k^*, \lambda_k^*)$ as in \cref{lemma:aux-ineq}. Then \cref{lemma:aux-ineq} with $\pi = \pi_{k}$ (note that $\pi$ need not satisfy the constraints) we have
    \begin{align*}
        0\leq& V^{\pi_k}(\ct_k, \pt_{k}) - V^{\pi_k^*}(\ct_k, \pt_{k}) + (\lambda_k^*)^T (V^{\pi_k}(\Dt_k, \pt_{k}) - \alpha),
    \end{align*}
    and by feasibility of $\pi_k^*$ we have $\Ltk (\pi_k^*, \lambda_k) \leq V^{\pi_k^*}(\ct_k, \pt_{k})$, with $\Ltk (\pi, \lambda)$ as defined in \cref{lemma:ialm-step}. Thus, the above becomes
    \begin{align}
        0\leq& V^{\pi_{k}}( \ct_k, \pt_{k}) - \Ltk (\pi_k^*, \lambda_k) + (\lambda_k^*)^T(V^{\pi_{k}}(\Dt_k, \pt_{k}) - \alpha). \label{eq:kkt-plugin}
    \end{align}
    Moreover, from \cref{lemma:ialm-step} we know that 
    \begin{align*}
        2\eta_k(V^{\pi_{k}}( \ct_k, \pt_{k}) + \lambda^T (V^{\pi_{k}}(\Dt_k, \pt_{k}) - \alpha)) + \| \lambda_{k+1} - \lambda \|^2 \\
        \leq 2\eta_k \Ltk (\pi, \lambda_k) + \| \lambda_k - \lambda \|^2 + 2\eta_k\epsilon_k,
    \end{align*} 
    and we can choose $(\pi, \lambda) = (\pi_k^*, \lambda_k^*)$ in this inequality to get
    \begin{align*}
        2\eta_k(V^{\pi_{k}}( \ct_k, \pt_{k}) + (\lambda_k^*)^T (V^{\pi_{k}}(\Dt_k, \pt_{k}) - \alpha)) + \| \lambda_{k+1} - \lambda_k^* \|^2 \\
        \leq 2\eta_k \Ltk (\pi_k^*, \lambda_k) + \| \lambda_k - \lambda_k^* \|^2 + 2\eta_k\epsilon_k.
    \end{align*} 
    Adding $2\eta_k$ times \cref{eq:kkt-plugin} to this and cancelling terms yields
    \begin{align*}
        \| \lambda_{k+1} - \lambda_k^* \|^2 \leq& \| \lambda_k - \lambda_k^* \|^2 + 2\eta_k\epsilon_k.
    \end{align*}
\end{proof}

From \cref{prop:dual-bound}, we readily obtain \cref{cor:dual-bound}.

\cordualbound*

\begin{proof}
    Taking square root and using $\sqrt{a+b} \leq \sqrt{a} + \sqrt{b}$ in \cref{prop:dual-bound} yields $\| \lambda_{k+1} - \lambda_k^* \| \leq \| \lambda_k - \lambda_k^* \| + \sqrt{2\eta_k\epsilon_k}$. By triangle inequality we get 
    \begin{align*}
        \| \lambda_{k+1} - \lambda_k^* \| \leq& \| \lambda_k - \lambda_k^* \| + \sqrt{2\eta_k\epsilon_k} \leq \| \lambda_k - \lambda_{k-1}^* \| + \| \lambda_{k-1}^* - \lambda_k^* \| + \sqrt{2\eta_k\epsilon_k},
    \end{align*}
    and thus by induction 
    \begin{align*}
        \| \lambda_{k+1} - \lambda_k^* \| \leq& \|\lambda_{K'+1} - \lambda_{K'}^*\| +  \sum_{t=K'+1}^{k} \| \lambda_{t-1}^* - \lambda_t^* \| + \sum_{t=K'+1}^{k} \sqrt{2\eta_t \epsilon_t}\\
        \leq& \|\lambda_{K'+1}\|  + \| \lambda_{K'}^*\| +  \sum_{t=K'+1}^{k} (\| \lambda_{t-1}^*\| + \|\lambda_t^* \|) + \sum_{t=K'+1}^{k} \sqrt{2\eta_t \epsilon_t}\\
        =& \|\lambda_k^*\| + 2 \sum_{t=K'}^{k-1} \| \lambda_t^*\| + \sum_{t=K'+1}^{k} \sqrt{2\eta_t \epsilon_t},
    \end{align*}
    again using the triangle inequality and that $\lambda_{K'+1} = 0$. The claim now follows by invoking the inverse triangle inequality on the LHS and rearranging. 
\end{proof}


\subsubsection{Bounding the Dual Maximizers} \label{sec:dual-max-bound}

We can now deduce \cref{thm:lstar-bound} from the preliminaries shown in \cref{sec:conv-prelim} as follows. See \cref{lemma:aux-ineq} from the previous section for a formal introduction of the primal-dual pairs $(\pi_k^*,\lambda_k^*)$.

\thmlstarbound*

\begin{proof}
    Conditioned on $G$, the conclusion from \cref{prop:pretraining} holds. That is, for every $k \in \curly{K', \dots, K}$ and every $i\in [I]$ we have 
    \begin{align*}
        V^{\pibar}(\dt_{i,k}, \pt_{k}) \leq \alpha_i-\nu\gamma,
    \end{align*}
    or equivalently, $\min_{i\in[I]} ( \alpha_i - V^{\pibar}(\dt_{i,k}, \pt_{k}) ) \geq \nu\gamma$. Moreover, since the expected costs are in $[0,1]$ and the time horizon is $H$, the value functions are in $[0,H]$, so we have $V^{\pibar}(\ct_k, \pt_{k}) - V^{\pi^*_k}(\ct_k, \pt_{k}) \leq H$. Hence Assumption \ref{ass:slater} holds, with $\pi^0 = \pibar$ and $\sigma = \frac{H}{\nu \gamma}$.

    As discussed in \cref{rmk:match-convex}, under Assumption \ref{ass:slater}, we can apply the general bound on dual maximizers from \cref{thm:conv-dual}. For this, we write \cref{eq:opt-KKT} as a convex optimization problem in the occupancy measure (see \cref{eq:opt-KKT-occ}), with $V^{\pi}(\ct_k, \pt_{k}) = \ct_k^T q^{\pi}(\pt_{k})$ and $V^{\pi}(\Dt_k, \pt_{k}) = \Dt_k q^{\pi}(\pt_{k})$. Then, set $X=Q(\pt_k)$, $\xbar = q^{\pibar}(\pt_{k})$, $f(\cdot) = \ct_{k}^T(\cdot)$ and $g_i(\cdot) = \dt_{i,k}^T(\cdot) - \alpha_i$ as in \cref{rmk:match-convex}. Plugging this into \cref{thm:conv-dual} indeed yields the second of the claimed inequalities since we have shown that Assumption \ref{ass:slater} holds with $\pi^0 = \pibar$, $\sigma = \frac{H}{\nu \gamma}$. The first inequality holds because $\lambda_k^*$ only has non-negative entries.
\end{proof}


\subsubsection{Optimization Error Bound}

We are now ready to prove the desired bound on the optimization errors.

\lemmaopterrork*

\begin{proof}
    Recall that conditioned on $G$, \cref{lemma:fdiff} shows ($k \in \curly{K'+1, \dots, K}$)
    \begin{align}
        [V^{\pi_{k}}( \ct_k, \pt_{k}) - V^{\pi^*}(c, p) ]_+ \leq& \epsilon_k + \frac{\| \lambda_{k} \|^2}{2\eta_k}, \label{eq:obj}\\
        [V^{\pi_{k}}(\dt_{i,k}, \pt_{k}) - \alpha_i]_+ \leq& \frac{\|\lambda_{k+1} \|}{\eta_k}, \label{eq:constr}
    \end{align}
    when dropping the negative terms in \cref{lemma:fdiff}. Moreover, from \cref{cor:dual-bound} and \cref{thm:lstar-bound} we get 
    \begin{align}
        \|\lambda_{k+1}\| \leq& ~2 \sum_{t=K'}^{k} \| \lambda_t^*\| + \sum_{t=K'+1}^{k} \sqrt{2\eta_t \epsilon_t} \nonumber\\
        \leq& ~2\sum_{t=K'}^{k} \sigma + \sum_{t=K'+1}^{k} \sqrt{2\eta_t \epsilon_t} \nonumber\\
        =& (2+2(k-K'))\sigma + \sum_{t=K'+1}^{k} \sqrt{2\eta_t \epsilon_t}, \label{eq:lbd-aux1}
    \end{align}
    where $\sigma = \frac{H}{\nu \gamma}$. Note that in the bound from \cref{eq:obj}, for $k=K'+1$ we have $\| \lambda_{K'+1} \|^2 = 0 \leq ((2+2 \cdot 1)\sigma + \sqrt{2\eta_{K'+1} \epsilon_{K'+1}})^2$ since $\lambda_{K'+1}=0$. And for $k \geq K' + 2$, by \cref{eq:lbd-aux1} we have $\| \lambda_{k} \|^2 \leq ((2+2((k-1)-K'))\sigma + \sum_{t=K'+1}^{k-1} \sqrt{2\eta_t \epsilon_t})^2 \leq ((2+2(k-K'))\sigma + \sum_{t=K'+1}^{k} \sqrt{2\eta_t \epsilon_t})^2$. Hence, the parameter-dependent bound on the objective errors follows by plugging this into \cref{eq:obj} and summing up. For the constraint errors, we can directly plug in the bound from \cref{eq:lbd-aux1} into \cref{eq:constr}, thus obtaining the second parameter-dependent bound.

    From this, we now show how to obtain bounds of order $O(\sqrt{K})$. Note that we can always choose $\epsilon_k := \frac{1}{2\eta_k}$ so that $\sum_{t=K'+1}^k \sqrt{2\eta_t \epsilon_t} = k-K' \leq \sigma (k-K')$, since $\sigma = \frac{H}{\nu \gamma} > 1$. If we do so, then we can loosely bound the term in \cref{eq:lbd-aux1} by 
    \begin{align}
        (2+2(k-K'))\sigma + \sum_{t=K'+1}^{k} \sqrt{2\eta_t \epsilon_t} \leq (2+3(k-K'))\sigma. \nonumber
    \end{align}
    As seen above, with this, \cref{eq:obj} and \cref{eq:constr} become
    \begin{align}
        [V^{\pi_{k}}( \ct_k, \pt_{k}) - V^{\pi^*}(c, p) ]_+ &\leq \epsilon_k + \frac{ ((2+3(k-K'))\sigma)^2 }{2\eta_k} \label{eq:fdiff-c},\\
        [V^{\pi_{k}}(\dt_{i,k}, \pt_{k}) - \alpha_i]_+ &\leq \frac{ (2+3(k-K'))\sigma }{\eta_k}. \label{eq:gdiff-d}
    \end{align}
    We can now choose\footnote{Note that if $\gamma$ is not known, having an estimate of $\sigma = \frac{H}{\nu \gamma}$ that is larger than the true value is clearly sufficient.} 
    \begin{align*}
        \eta_{K'+k} &:= ((2+3k)\sigma)^{2.5} \tab (k>0)
    \end{align*}
    and sum each of the two errors above to bound the regret due to optimization. For the objective cost, we get (recall that we had set $\epsilon_k = 1/(2\eta_k)$)
    \begin{align*}
        &\sum_{k=K'+1}^{K} [V^{\pi_{k}}( \ct_k, \pt_{k}) - V^{\pi^*}(c, p) ]_+ \\
        &\overset{\cref{eq:fdiff-c}}{\leq} \sum_{k=K'+1}^{K} \epsilon_k + \frac{1}{2}  \sum_{k=K'+1}^{K} \frac{( (2+3(k-K'))\sigma)^2}{((2+3(k-K'))\sigma)^{2.5}}\\
        &\leq \frac{1}{2}  \sum_{k=1}^{K} \frac{1}{( (2+3k)\sigma)^{2.5}}  + \frac{1}{2}  \sum_{k=1}^{K} \frac{1}{\sqrt{(2+3k)\sigma}}\\
        &= O(K^{1/2} ), 
    \end{align*}
    up to absolute constants, where the last step follows by considering the dominating series $\sum_{k\geq 1} 1/k^2 = \pi^2 / 6$ and using $\sum_{k=1}^K 1/\sqrt{k} = \Theta(\sqrt{K})$, as well as $\sigma > 1$. Similarly, for the constraint cost, we get
    \begin{align*}
        \sum_{k=K'+1}^{K} [V^{\pi_{k}}(\dt_{i,k}, \pt_{k}) - \alpha_i]_+ \leq&  \sum_{k=K'+1}^{K} \frac{ ((2+3(k-K'))\sigma)}{ ( (2+3k)\sigma)^{2.5} }\\
        \leq& \sum_{k=1}^{K} \frac{ 1}{ ( (2+3k)\sigma)^{1.5} } \\
        =& O(K^{1/2}),
    \end{align*}
    up to absolute constants. 
\end{proof}

We remark that the choices for $\eta_k$ and $\epsilon_k$ are not necessarily optimal, and there is a trade-off between the regret due to the optimization error and the iteration complexity of the inner loop, as we discuss in \cref{rmk-main-param}.

\end{document}